\documentclass[11pt]{article} 
\usepackage{fullpage}
\usepackage[backref,colorlinks,citecolor=blue,bookmarks=true]{hyperref}
\usepackage{url}
\usepackage{amsmath,amssymb,amsfonts,amsthm,color}
\usepackage{subfigure,graphicx, tikz,enumitem}
\usepackage{caption}
\definecolor{darkgreen}{rgb}{0.04,0.63,0.34}
\usepackage{natbib}
\usepackage{algorithm}
\usepackage{float}
\usepackage{algorithmic}

\newcommand{\Medn}{M_{\epsilon,\delta}(n)}
\newcommand{\Redn}{R_{\epsilon,\delta}(n)}

\newtheorem{Theorem}{Theorem}
\newtheorem{Theorem*}{Theorem}

\newtheorem{Claim*}[Theorem]{Claim}
\newtheorem{Corollary}[Theorem]{Corollary}

\newtheorem{CounterExample*}{$\overline{\hbox{\bf Example}}$}

\newtheorem{Example*}[Theorem]{Example}

\newtheorem{Intuition*}[Theorem]{Intuition}
\newtheorem{Joke*}[Theorem]{Joke}
\newtheorem{Lemma}[Theorem]{Lemma}
\newtheorem{Lemma*}[Theorem]{Lemma}
\newtheorem{Open problem}[Theorem]{Open problem}

\newtheorem{Question*}[Theorem]{Question}

\def \bSubexa    {\begin{subexa}}

\newcommand{\ignore}[1]{}

\def \cO     {{\cal O}}

\definecolor{light}{gray}{.75}

\def \upto  {{,}\ldots{,}}

\def \Sets#1{{\left\{#1\right\}}}

\def \ceil#1{{\lceil{#1}\rceil}}

\def \floor#1{{\lfloor{#1}\rfloor}}

\def \Paren#1{{\left({#1}\right)}}

\newcommand{\ed}{\stackrel{\mathrm{def}}{=}}

\def \half    {{\frac12}}

\def\ignore#1{}

\newcommand{\bi}{\begin{itemize}}
\newcommand{\ei}{\end{itemize}}

\def\orpro{\mathop{\mathchoice
   {\vee\kern-.49em\raise.7ex\hbox{$\cdot$}\kern.4em}
   {\vee\kern-.45em\raise.63ex\hbox{$\cdot$}\kern.2em}
   {\vee\kern-.4em\raise.3ex\hbox{$\cdot$}\kern.1em}
   {\vee\kern-.35em\raise2.2ex\hbox{$\cdot$}\kern.1em}}\limits}

\def\andpro{\mathop{\mathchoice
 {\wedge\kern-.46em\lower.69ex\hbox{$\cdot$}\kern.3em}
 {\wedge\kern-.46em\lower.58ex\hbox{$\cdot$}\kern.25em}
 {\wedge\kern-.38em\lower.5ex\hbox{$\cdot$}\kern.1em}
 {\wedge\kern-.3em\lower.5ex\hbox{$\cdot$}\kern.1em}}\limits}

\def\simge{\mathrel{%
   \rlap{\raise 0.511ex \hbox{$>$}}{\lower 0.511ex \hbox{$\sim$}}}}

\def\simle{\mathrel{
   \rlap{\raise 0.511ex \hbox{$<$}}{\lower 0.511ex \hbox{$\sim$}}}}

\newcommand{\BSR}{\textsc{Binary-Search-Ranking}}
\newcommand{\IBR}{\textsc{Interval-Binary-Search}}
\newcommand{\BR}{\textsc{Binary-Search}}
\newcommand{\BFBR}{\textbf{\textsc{Binary-Search}}}
\newcommand{\BFIBR}{\textbf{\textsc{Interval-Binary-Search}}}
\newcommand{\RankX}{\textsc{Rank-$x$}}

\newcommand{\ALOOP}[1]{\ALC@it\algorithmicloop\ #1%
  \begin{ALC@loop}}
\newcommand{\ENDALOOP}{\end{ALC@loop}\ALC@it\algorithmicendloop}

\begin{document}
\title{
Maximum Selection and Ranking under Noisy Comparisons
}
\author{
\begin{tabular}[t]{c@{\extracolsep{10em}}c}
  Moein Falahatgar & Alon Orlitsky\\
 UC San Diego & UC San Diego\\ 
\small \texttt{moein@ucsd.edu} & \small \texttt{alon@ucsd.edu}
\end{tabular}
\vspace{2ex}\\
\begin{tabular}[t]{c@{\extracolsep{6.5em}}c}
   Venkatadheeraj Pichapati & Ananda Theertha Suresh\\
  UC San Diego & Google Research\\
 \small \texttt{dheerajpv7@ucsd.edu} & \small\texttt{theertha@google.com} 
\end{tabular}
\vspace{2ex}\\  
}
\maketitle
\begin{abstract}
We consider $(\epsilon,\delta)$-PAC maximum-selection
and ranking for general probabilistic models whose
comparisons probabilities satisfy strong stochastic transitivity
and stochastic triangle inequality.
Modifying the popular knockout tournament, we propose 
a maximum-selection algorithm that uses 
$\cO\Paren{\frac{n}{\epsilon^2} \log \frac1{\delta}}$ comparisons,
a number tight up to a constant factor. 
We then derive a general framework that improves the performance
of many ranking algorithms, and combine it 
with merge sort and binary search to obtain a
ranking algorithm that uses 
$\cO\Paren{\frac{n\log n (\log \log n)^3}{\epsilon^2}}$ comparisons for
any $\delta\ge\frac1n$, a number optimal up to a $(\log \log n)^3$ factor.
\end{abstract}

\section{Introduction}
\label{sec:intro}
\subsection{Background}
Maximum selection and sorting using pairwise comparisons are
computer-science staples taught in most introductory classes and used
in many applications. In fact, sorting, also known as \emph{ranking},
has been claimed to utilize 25\% of computer cycles worldwide~\cite{book}.

In many applications, the pairwise comparisons produce only random
outcomes. For example, sports tournaments rank teams based on pairwise
matches, but match outcomes are probabilistic in nature. Patented by
Microsoft, \emph{TrueSkill}~\cite{Ralf06} is such a ranking system for Xbox gamers.
Another important application is online advertising. Prominent web
pages devote precious little space to advertisements, limiting
companies like Google, Microsoft, or Yahoo!  to present a typical user
with just a couple of ads, of which the user selects at most
one. Based on these small random comparisons, the company would like
to rank the ads according to their appeal~\cite{FilipT07, FilipMT08}.

This and related applications have brought about a resurgence of
interest in maximum selection and ranking using noisy comparisons.
Several noise models were considered, including the popular 
Plackett-Luce model~\cite{RobinL75,Ducan05}.
Yet even for such specific models, the complexity of maximum selection
was known only up to a $\log n$ factor and the complexity of
ranking was known only up to a $\log n$ factor. We consider
a broader class of models 
and propose algorithms that are optimal up to a constant factor for 
maximum selection and up to $(\log\log n)^3$ for ranking.


\subsection{Notation}
Noiseless comparison assumes an unknown underlying ranking $r(1)\upto
r(n)$ of the elements such that if two elements are compared, the
higher-ranked one is selected.  Similarly for noisy comparisons, we
assume an unknown ranking of the $n$ elements, but now if two elements
$i$ and $j$ are compared, $i$ is chosen with some unknown probability
$p(i,j)$ and $j$ is chosen with probability $p(j,i) = 1-p(i,j)$, where
the higher-ranked element has probability $\ge\half$.  Repeated
comparisons are independent of each other.

Let $\tilde{p}(i,j)=p(i,j)-\frac12$ reflect the additional probability
by which $i$ is preferable to $j$. Note that $\tilde{p}(j,i) =
-\tilde{p}(i,j)$ and $\tilde{p}(i,j) \ge 0$ if $r(i) > r(j)$. $|\tilde{p}(i,j)|$ can also be seen as a measure of
dissimilarity between $i$ and $j$. In our model we assume that two
very natural properties hold whenever $r(i)>r(j)>r(k)$.

(1)\emph{Strong stochastic transitivity:} 
\[
\tilde{p}(i,k) \ge \max(\tilde{p}(i,j), \tilde{p}(j,k));
\]
(2) \emph{Stochastic triangle inequality:}
\[
 \tilde{p}(i,k) \le \tilde{p}(i,j) + \tilde{p}(j,k).
\]
These properties are satisfied by several popular preference models 
e.g., Plackett-Luce(PL) model.

Two types of algorithms have been proposed for finding the maximum and
ranking under noisy comparisons: \emph{non-adaptive} or \emph{offline}~\cite{ArunS14,SahandSD12,Negahban16,Jang16} where we cannot
choose the comparison pairs, and \emph{adaptive} or \emph{online}
where the comparison pairs are selected sequentially based on previous
results. In this paper we focus on the latter.

We specify the desired output via the $(\epsilon,\delta)$-PAC
paradigm~\cite{YisongT11,RobertBE14} that requires the output to
likely closely approximate the intended value.  Specifically, given
$\epsilon,\delta>0$, with probability $\ge 1- \delta$, maximum
selection must output an element $i$ such that for $j$ with $r(j) =
n$,
\[
p(i,j) \geq \frac{1}{2} - \epsilon.
\]
We call such an output $\epsilon$-maximum.  Similarly, with
probability $\ge 1- \delta$, the ranking algorithm must output a
ranking $r'(1)\upto r'(n)$ such that whenever $r'(i)>r'(j)$,
\[
p(i,j) \geq \frac{1}{2} - \epsilon.
\]
We call such a ranking $\epsilon$-ranking.

\ignore{
Let $\Medn$ and $\Redn$, respectively, denote the least number of
comparisons in the worst-case scenario used by any algorithm for PAC
maximum selection and ranking of $n$ elements.
With noiseless comparisons, selecting the maximum of $n$ elements is
easily seen to require $n-1$ pairwise comparisons. Ranking $n$
elements requires $\Theta(n\log n)$ pairwise comparisons, achievable
for example by merge-sort.
Since noisy comparators are less informative than their noiseless
counterparts, it is intuitively clear, and indeed true, that
$\Medn=\Omega(n)$ and $\Redn=\Omega(n\log n)$.  Over the past decade,
a number of researchers have tried to achieve these noiseless lower
bounds using noisy comparisons.
}

\subsection{Paper outline}
In Section~\ref{sec:related_work} we mention somerelated
works. In Section~\ref{sec:results} we highlight our main
contributions. In Section~\ref{sec:max} we propose the maximum
selection algorithm. In Section~\ref{sec:rank} we propose the ranking
algorithm. In Section~\ref{sec:experiments} we provide experiments.
In Section~ \ref{sec:conclusion} we
discuss the results and mention some future directions.

\section{Related work}
\label{sec:related_work}
~\cite{HeckelSRW16,TanguyFRS13,RobertBE14,RobertBE14} assume no
underlying ranking or constraints on probabilities and find ranking
based on Copeland, Borda count and Random Walk procedures.
~\cite{TanguyFRS13,RobertBE14} showed that if the probabilities
$p(i,j)$ are not constrained, both maximum selection and ranking
problems require $\Theta(n^2)$ comparisons.  Several models have
therefore been considered to further constrain the probabilities.

Under the assumptions of strong stochastic transitivity and triangle
inequality ,~\cite{YisongT11} derived a PAC maximum selection
algorithm that uses $O\Paren{\frac{n}{\epsilon^2} \log
  \frac{n}{\epsilon \delta}}$ comparisons.  \ignore{ Their algorithm
  selects a random element $i$, compares it with all other elements
  $j$, and updates the estimates of $p(i,j)$. It therefore uses
  $\Theta(n^2)$ memory to maintain all probabilities values.  }
~\cite{BalazsRAE15} derived a PAC ranking algorithm for PL-model
distributions that requires $O(\frac{n}{\epsilon^2}\log
n\log\frac{n}{\delta\epsilon})$ comparisons.

In addition to PAC paradigm,~\cite{YisongT11} also considered this problem
under the bandit setting and bounded the regret of the resulting
dueling bandits problem.  Following this work, several other works
e.g. ~\cite{VasilisAR16} looked at similar formulation.

Another non-PAC approach by ~\cite{RobertEB14, UrielPDE94}
solves the maximum selection and ranking problems. They assume
a lower bound on $|\tilde{p}(i,j)|$ and the number of comparisons
depends on this lower bound. If $|\tilde{p}(i,j)|$ is $0$ for any pair,
then these algorithms will never terminate.

Several other noise models have also been considered in practice that
have either adverserial noise or the stochastic noise that does not
obey triangle inequality. For example,
~\cite{JayadevAAT14,AcharyaFJOS16,JayadevAAT214} considered
adversarial sorting with applications to density estimation
and~\cite{MiklosVAJ15} considered the same with deterministic
algorithms. Mallows stochastic model~\cite{RobertEB14} does not
satisfy the stochastic triangle inequality and hence our theoretical
guarantees do not hold under this model. However our simulations
suggest that our algorithm can have a reasonable performance over
Mallows model. \ignore{We can think of this ranking problem as a
  conditional sampling~\cite{ClemetDR, JayadevCG, MoeinAAVT15} setup
  where we are allowed to query a result of competition between two
  alternatives.}
\section{New results}
\label{sec:results}
Recall that we study $(\epsilon,\delta)$-PAC model for the problems of
online maximum selection and ranking using pairwise comparisons under
strong stochastic transitivity and stochastic triangle inequality
assumptions. The goal is to find algorithms that use small number of
comparisons. Our main contributions are:
\begin{itemize}
\item A maximum selection algorithm that uses 
  $O\Paren{\frac{n}{\epsilon^2}\log \frac1{\delta}}$ comparisons and therefore our
  algorithm is optimal up to constants.
\ignore{
\item The proposed algorithm for the maximum selection is memory- and
  time-efficient.  It uses $\cO(n)$ memory space and it can be
  parallelized to run in time sublinear in the input size.
}
\item A ranking algorithm that uses at
  most $O\Paren{\frac{n(\log n)^3}{\epsilon^2} \log
    \frac{n}{\delta}}$ comparisons and outputs
  $\epsilon$-ranking for any $\delta$.
\item A framework that given any ranking algorithm with
  $\cO\Paren{\frac{n (\log n)^x}{\epsilon^2} \log \frac{n}{\delta}}$
  sample complexity, provides a ranking algorithm with $\cO\Paren{\frac{n
      \log n (\log \log n)^x}{\epsilon^2}}$ sample complexity for
  $\delta \ge \frac1{n}$.
\item Using the framework above, we present an algorithm that uses
  at most $O\Paren{\frac{n \log n(\log \log n)^3}{\epsilon^2}}$
  comparisons and outputs $\epsilon$-ranking for
  $\delta=\frac1{n}$. We also show a lower bound of
  $\Omega\Paren{\frac{n\log n}{\epsilon^2}}$ on the number of
  comparisons used by any PAC ranking algorithm, therefore proving our
  algorithm is optimal up to $\log \log n$ factors.
\end{itemize}
\ignore{
Combining these results, we obtain:
\[
\Medn=\Theta\Paren{\frac{n}{\epsilon^2} \log \frac1{\delta}} \text{
  and }R_{\epsilon,\frac1n}(n)=\tilde{\Theta}\Paren{\frac{n \log
    n}{\epsilon^2}}.
\]}
\ignore{
Also, unlike some of the previous algorithms,
all our algorithms use only liner space.
Note also that our algorithms are \emph{adaptive} in that each
comparison is determined by the answers to previous ones.
}

\section{Maximum selection}
\label{sec:max}
%
\subsection{Algorithm outline}
We propose a simple maximum-selection algorithm based on Knockout
tournaments.  Knockout tournaments are often used to find a maximum
element under non-noisy comparisons. Knockout tournament of $n$
elements runs in $\ceil{\log n}$ rounds where in each round it
randomly pairs the remaining elements and proceeds the winners to next
round.

Our algorithm, given in \textsc{Knockout} uses
$\cO\Paren{\frac{n}{\epsilon^2}\log \frac1{\delta}}$ comparisons and
$\cO(n)$ memory to find an $\epsilon$-maximum. \cite{YisongT11} uses
$\cO\Paren{\frac{n}{\epsilon^2} \log \frac{n}{\epsilon\delta}}$
comparisons and $\cO(n^2)$ memory to find an $\epsilon$-maximum. Hence
we get $\log n$-factor improvement in the number of comparisons and also we
use linear memory compared to quadratic memory. 
Using the lower bound in ~\cite{UrielPDE94}, it can be inferred that
the best PAC maximum selection algorithm requires 
$\Omega\Paren{\frac{n}{ \epsilon^2} \log \frac1{\delta}}$ comparisons,
hence up to constant factor, \textsc{Knockout} is optimal. Also our
algorithm can be parallelized to run in $\cO\Paren{\frac{\log
    n}{\epsilon^2} \log \frac{1}{\delta}}$ time.

Due to the noisy nature of the comparisons, we repeat each comparison
several times to gain confidence about the winner. Note that in
knockout tournaments, the
number of pairs in a round decreases exponentially with each
round. Therefore we afford to repeat the comparisons more
times in the latter rounds and get higher confidence.  Let $b_i$ be
the highest-ranked element (according to unobserved underlying
ranking) at the beginning of round $i$.  We repeat the comparisons in
round $i$ enough times to ensure that $\tilde{p}(b_{i}, b_{i+1}) \le
\frac{c\epsilon}{ 2^{i/3}}$ with probability $\ge
1-\frac{\delta}{2^i}$ where $c=2^{1/3}-1$. By the stochastic triangle
inequality, $\tilde{p}(b_1,b_{\ceil{\log n}+1}) \le
\sum_{i=1}^{\ceil{\log n}+1} \frac{c\epsilon}{2^{i/3}} \le \epsilon$
with probability $\ge 1-\delta$.

There is a relaxed notion of \emph{strong stochastic transitivity}. For
$\gamma\ge1$, \emph{$\gamma$-stochastic
  transitivity}~\cite{YisongT11}: if $r(i)> r(j) >r(k)$, then $
\max(\tilde{p}(i,j), \tilde{p}(j,k))\le \gamma \cdot \tilde{p}(i,k)$.

Our results apply to this general notion of $\gamma$-stochastic
transitivity and the analysis of \textsc{Knockout} is
presented under this model.~\cite{YisongT11} uses
$\cO(\frac{n\gamma^6}{\epsilon^2} \log \frac{n}{\delta})$ comparisons
to find an $\epsilon$-maximum whereas \textsc{Knockout} uses only
$\cO(\frac{n \gamma^2}{\epsilon^2} \log \frac1{\delta})$
comparisons. Hence we get a huge improvement in the exponent of
$\gamma$ as well as removing the extra $\log n$ factor.

To simplify the analysis, we assume that $n$ is a power of 2,
otherwise we can add $2^{\ceil{\log n}} - n$ dummy elements that lose
to every original element with probability 1.  Note that all
$\epsilon$-maximums will still be from the original set.
\subsection{Algorithm}
We start with a subroutine \textsc{Compare} that compares
two elements. It compares two elements $i$, $j$ and
maintains empirical probability $\hat{p}_i$, a proxy for $p(i,j)$. It
also maintains a confidence value $\hat{c}$ s.t., w.h.p.,
$\hat{p}_i \in (p(i,j) - \hat{c} , p(i,j) + \hat{c})$. \textsc{Compare}
stops if it is confident about the winner or if it reaches 
its comparison budget $m$. If it reaches $m$ comparisons, it outputs 
the element with more wins breaking ties randomly.
\begin{algorithm}
\caption{\textsc{Comprare}}
\textbf{Input:} element $i$, element $j$, bias $\epsilon$, confidence $\delta$.
\\\textbf{Initialize:} $\hat{p_i} = \frac12$, $\hat{c} = \frac12$, $m = \frac1{2\epsilon^2}
\log\frac{2}{\delta}$, $r = 0$, $w_i = 0$.
\begin{enumerate}
\item \textbf{while} ($|\hat{p_i} - \frac12| \le \hat{c} - \epsilon$ and $r \le m$)
\begin{enumerate}
\item Compare $i$ and $j$. \textbf{if} $i$ wins $w_i$ = $w_i+1$.
\item $r = r+1$, $\hat{p_i} = \frac{w_i}{r}$, $\hat{c} =                                      
  \sqrt{\frac{1}{2r}\log\frac{4r^2}{\delta}}$.
\end{enumerate}
\end{enumerate}
\textbf{if} $\hat{p_i} \le \frac12$ \textbf{Output:} j. \textbf{else}
\textbf{Output:} i.
\end{algorithm}

We show that the subroutine \textsc{Compare} always outputs the correct
winner if the elements are well seperated. 
\ignore{
\begin{Lemma}
\label{lem:chernoff1}
If $\tilde{p}(i,j) \ge \epsilon$ and $\{i,j\}$ are compared for $\frac{1}{2\epsilon^2}\log \frac1{\delta}$ times 
then with probability at least $1-\delta$, $i$ will win at least as many times as $j$ does.
\end{Lemma}
\begin{proof}
Using chernoff bound,
\begin{align*}
Pr(i\text{ wins less times than }j) 
&\le e^{-\frac{1}{2\epsilon^2} \log \frac1{\delta} (2 \epsilon^2)}\\
&\le e^{-\log \frac1{\delta}}= \delta.\qedhere
\end{align*}
\end{proof}
}
\begin{Lemma}
\label{lem:chernoff1}
If $\tilde{p}(i,j) \ge \epsilon$, then
\[
Pr(\textsc{Compare}(i,j,\epsilon,\delta) \ne i) \le \delta.
\] 
\end{Lemma}
Note that instead of using fixed number of comparisons, \textsc{Compare} stops
the comparisons adaptively if it is confident about the winner.  If
$|\tilde{p}(i,j)| \gg \epsilon$, \textsc{Compare} stops much before
comparison budget
$\frac{1}{2\epsilon^2}\log\frac{2}{\delta}$ and hence works 
better in practice. 

Now we present the
subroutine \textsc{Knockout-Round} that we use in main algorithm
\textsc{Knockout}.
\subsubsection{Knockout-Round}
 \textsc{Knockout-Round} takes a set $S$ and outputs a set of size
 $|S|/2$. It randomly pairs elements, compares each pair using
 \textsc{Compare}, and returns the set of winners. We will later show
 that maximum element in the output set will be comparable to maximum
 element in the input set.

\ignore{
\begin{algorithm}
\caption{\textsc{Knockout-Round}}
\textbf{Input:} Set $S$, bias $\epsilon$, confidence $\delta$.
\begin{enumerate}
\item
Pair elements in $S$ randomly.
\item
Compare elements in each pair $\frac{1}{2\epsilon^2}\log \frac{1}{\delta}$ times.
\item
For each pair, declare the element that won more times as winner. Resolve ties randomly.
\end{enumerate}
\textbf{Output:} The set of all winners.
\end{algorithm}
}

\begin{algorithm}
\caption{\textsc{Knockout-Round}}
\textbf{Input:} Set $S$, bias $\epsilon$, confidence $\delta$.
\\\textbf{Initialize:} Set $O = \emptyset$.
\begin{enumerate}
\item Pair elements in $S$ randomly.
\item \textbf{for} every pair $(a,b)$: 
\begin{enumerate}
\item Add \textsc{Compare}$(a,b,\epsilon,\delta)$ to $O$.
\end{enumerate}
\end{enumerate}
\textbf{Output:} $O$
\end{algorithm}

Note that comparisons between each pair can be handled by a different processor and hence this algorithm 
can be easily parallelized.

\ignore{
\begin{center}
\fbox{\begin{minipage}{1.0\textwidth}
Algorithm \textsc{Knockout-Round}\newline
\textbf{Input:} Set $S$, bias $\epsilon$, confidence $\delta$.
\begin{enumerate}
\item
Pair elements in $S$ randomly.
\item
Compare elements in each pair $\frac{1}{2\epsilon^2}\log \frac{1}{\delta}$ times.
\item
For each pair, declare the element that won for more times as winner. Resolve ties randomly.
\end{enumerate}
\textbf{Output:} The set of all winners.
\end{minipage}}
\captionof{figure}{\textsc{Knockout-Round}}
\end{center}
}

Note that a set $S$ can have several maximum elements. Comparison
probabilities corresponding to all maximum elements will be essentially
same because of triangle inequality. We define $\text{max}(S)$
to be the maximum element with the least index, namely,
\[
\text{max}(S) \ed S\Big({\min\{i : \tilde{p}(S(i),S(j)) \ge 0\quad \forall j\}}\Big).
\]

\begin{Lemma}
\label{lem:knockout_round}
$\textsc{Knockout-Round}(S, \epsilon, \delta)$ uses
$\frac{|S|}{4\epsilon^2} \log \frac2{\delta}$ comparisons and with
probability $\ge 1-\delta$,
\[
\tilde{p}\Bigg(\text{max}(S),
\text{max}\Big(\textsc{Knockout-Round}(S,\epsilon, \delta)\Big)\Bigg) \le \gamma
\epsilon
\]
\end{Lemma}

\subsubsection{\textsc{Knockout}}
Now we present the main algorithm \textsc{Knockout}.
\textsc{Knockout} takes an input set $S$ and runs
$\log n$ rounds of \textsc{Knockout-Round} halving 
the size of $S$ at the end of each round. Recall that
\textsc{Knockout-Round} makes sure that maximum element 
in the output set is comparable to maximum element in 
the input set. Using this, \textsc{Knockout} makes sure that the
output element is comparable to maximum element in 
the input set. 

Since the size of $S$ gets halved after each round, \textsc{Knockout}
compares each pair more times in the latter rounds. Hence the bias
between maximum element in input set and maximum element in output set
is small in latter rounds.
  
\begin{algorithm}
\caption{\textsc{Knockout}}
\textbf{Input:} Set $S$, bias $\epsilon$, confidence $\delta$,
stochasticity $\gamma$.\newline \textbf{Initialize:} $i=1$, $S =$ set
of all elements, $c = 2^{1/3}-1$.\newline \textbf{while $|S| > 1$}
\begin{enumerate}
\item
$S$ = \textsc{Knockout-Round}$\Paren{S,  \frac{c\epsilon}{\gamma2^{i/3}}, \frac{\delta}{2^i}}$.
\item
$i = i+1.$
\end{enumerate}
\textbf{Output:} the unique element in $S$.
\end{algorithm}

\ignore{
\begin{center}
\fbox{\begin{minipage}{1.0\textwidth}
Algorithm \textsc{Knockout}\newline
\textbf{Input:} Set $S$, bias $\epsilon$, confidence $\delta$.\newline
\textbf{Initialize:} $i=1$, $S =$ set of all elements, $c = 2^{1/3}-1$.\newline
\textbf{repeat}
\begin{enumerate}
\item
$S$ = \textsc{Knockout-Round}$\Paren{S,  \frac{c\epsilon}{\gamma2^{i/3}}, \frac{\delta}{2^i}}$.
\item
$i = i+1.$
\end{enumerate}
\textbf{while} $|S| > 1$.\newline
\textbf{Output:} the unique element in $S$.
\end{minipage}}
\end{center}
\subsection{Analysis}
}

Note that  \textsc{Knockout} uses only memory of set $S$ and hence $\cO(n)$ memory suffices.
Now we bound the number of comparisons used by \textsc{Knockout} and 
prove the correctness.

\begin{Theorem}
\textsc{Knockout}$(S, \epsilon, \delta)$ uses
$\cO\Paren{\frac{\gamma^2|S|}{\epsilon^2} \log \frac1{\delta}}$
comparisons and with probability at least $1-\delta$, outputs an $\epsilon$-maximum.
\label{thm:max}
\end{Theorem}

\ignore{
wins against the best with probability of at least $\frac12-\epsilon$.
Let the best remaining element at step $i$ be $b_i$ and set $S_i$ represent the set of all elements $k$ that have $p(b_i,k) > \frac12 +\frac{\epsilon}{11\cdot(1.1)^i}$. At step $i$ in \textsc{Knockout} with probability $1-\frac{\delta}{2^i}$, $b_i$ can win against any $k \in S_i$. Therefore, with probability $1-\frac{\delta}{2^i}$, $p(b_i, b_{i+1}) < \frac12 + \frac{\epsilon}{11\cdot(1.1)^i}$.  
Let $E_i = 1_{p(b_i, b_{i+1}) <\frac12 + \frac{\epsilon}{11\cdot(1.1)^i}}$  . As we showed above, $Pr(E_i =0) \leq \frac{\delta}{2^i}$. Therefore, using union bound, $Pr(\prod_{i=1}^{\infty}E_i = 1) \geq 1-\sum_{i=1}^{\infty}\frac{\delta}{2^i} = 1- \delta$. Hence using triangle inequality of probabilities, with probability $1-\delta$, $p(b_{1}, b_{\infty}) \leq \frac12 + \sum_{i=1}^{\infty}(p(b_i, b_{i+1}) - \frac12) \leq \frac12 + \sum_{i=1}^{\infty}\frac{\epsilon}{11\cdot (1.1)^i} = \frac12 + \epsilon.$
}


\section{Ranking}
\label{sec:rank}
We propose a ranking algorithm that with probability at least
$1-\frac1{n}$ uses $\cO\Paren{\frac{n\log n (\log \log n)^3}{\epsilon^2} } $
comparisons and outputs an $\epsilon$-ranking.

Notice that we use only $\tilde{\cO}\Paren{\frac{n \log                                                                           
    n}{\epsilon^2}}$ comparisons for $\delta = \frac1{n}$ where as
~\cite{BalazsRAE15} uses $\cO\Paren{n (\log n)^2/ \epsilon^2}$
comparisons even for constant error probability $\delta$. Furthermore
~\cite{BalazsRAE15} provided these guarantees only under Plackett-Luce
model which is more restrictive compared to ours.  Also, their
algorithm uses $\cO(n^2)$ memory compared to $\cO(n)$ memory
requirement of ours.

Our main algorithm \textsc{Binary-Search-Ranking} assumes the		
existence of a ranking algorithm \RankX\ that with probability
at least $1-\delta$ uses $ \cO\Paren{\frac{n}{\epsilon^2} (\log n)^x
  \log \frac{n}{\delta}} $ comparisons and outputs an
$\epsilon$-ranking for any $\delta > 0$, $\epsilon > 0$ and some
$x>1$.  We also present a \RankX\ algorithm with $x=3$. 

Observe that we need \RankX\ algorithm to work for any model that
satisfies strong stochastic transitivity and stochastic triangle
inequality.~\cite{BalazsRAE15} showed that their algorithm works for
Plackett-Luce model but not for more general model. So we present a
\RankX\ algorithm that works for general model.

The main algorithm \textsc{Binary-Search-Ranking} randomly selects
$\frac{n}{(\log n)^x}$ elements (anchors) and rank
them using \RankX\ . The algorithm has then effectively created
$\frac{n}{(\log n)^x}$ bins, each between two successively ranked
anchors. Then for each element, the algorithm identifies the bin it
belongs to using a noisy binary search algorithm.  The algorithm then
ranks the elements within each bin using \RankX\ .

We first present \textsc{Merge-Rank}, a \textsc{Rank-3} algorithm.

\subsection{Merge Ranking}
\label{sec:merge}
We present a simple ranking algorithm \textsc{Merge-Rank} that uses
$\cO\Paren{\frac{n(\log n)^3}{\epsilon^2}  \log \frac{n}{\delta}}$
comparisons, $O(n)$ memory and with probability $\ge 1-\delta$ outputs
an $\epsilon$-ranking.  Thus \textsc{Merge-Rank} is a \RankX\ 
algorithm for $x=3$.

Similar to Merge Sort, \textsc{Merge-Rank} divides the elements into
two sets of equal size, ranks them separately and combines the sorted
arrays. Due to the noisy nature of comparisons, \textsc{Merge-Rank}
compares two elements $i,j$ sufficient times, so that the comparison
output is correct with high probability when $|\tilde{p}(i,j)| \ge
\frac{\epsilon}{\log n}$. Put differently, \textsc{Merge-Rank}
is same as the typical Merge Sort, except it uses \textsc{Compare} as 
the comparison function. 

Let's define the error of an ordered set $S$ as the 
maximum distance between two wrongly ordered items in $S$, namely,
\[
err(S) \ed \max_{1\le i \le j \le |S|} \Big({\tilde{p}(S(j), S(i))}\Big).
\]
We show that when we merge two ordered sets, the error of the
resulting ordered set will be at most $\frac{\epsilon}{\log n}$ more
than the maximum of errors of individual ordered sets.

Observe that \textsc{Merge-Rank} is a recursive algorithm and the
error of a singleton set is $0$. Two singleton sets each containing a
unique element from the input set merge to form a set with two
elements with an error at most $\frac{2\epsilon}{\log n}$, then two
sets with two elements merge to form a set with four elements with an
error of at most $\frac{3\epsilon}{\log n}$ and henceforth.  Therefore
the error of the output ordered set is bounded by $\epsilon$.

\ignore{
\begin{algorithm}
\caption{\textsc{Compare}}
\textbf{Input:} elements $a$ and $b$, number of comparisons $k$.
\begin{enumerate}
\item Compare $a$ and $b$ for $k$ times and return the fraction of times
  $a$ wins over $b$.
\end{enumerate}
\end{algorithm}
}
Lemma~\ref{lem:merge_rank} shows that $\textsc{Merge-Rank}$ can output
an $\epsilon$-ranking of $S$ with probability $\ge 1-\delta$. It also
bounds the number of comparisons used by the algorithm.
\begin{Lemma}
\label{lem:merge_rank}
$\textsc{Merge-Rank}\Paren{S,\frac{\epsilon}{\log |S|},
  \frac{\delta}{|S|^2}}$ takes $\cO\Paren{\frac{|S|(\log
    |S|)^3}{\epsilon^2} \log \frac{|S|}{\delta}}$ comparisons and with
probability $\ge 1 - \delta$, outputs an $\epsilon$-ranking. Hence,
\textsc{Merge-Rank} is a \textsc{Rank-3} algorithm.
\end{Lemma}

Now we present our main ranking algorithm.

\subsection{\BSR}
\label{sec:BSR}
We first sketch the algorithm outline below. We then provide a proof
outline.

\subsubsection{Algorithm outline}
Our algorithm is stated in \textsc{Binary-Search-Ranking}. It can be
summarized in three major parts.

\textbf{Creating anchors:} (Steps $1$ to $3$)
\textsc{Binary-Search-Ranking} first selects a set $S'$ of
$\frac{n}{(\log n)^x}$ random elements (anchors) and ranks them
using \RankX\ . At the end of this part, there are
$\frac{n}{(\log n)^x}$ ranked anchors. Equivalently, the algorithm
creates $\frac{n}{(\log n)^x}-1$ bins, each bin between two
successively ranked anchors.

\textbf{Coarse ranking:} (Step $4$) After forming the bins, 
the algorithm uses a random walk on a binary search tree, to
find which bin each element belongs to. \IBR\ is similar to the noisy
binary search algorithm in~\cite{UrielPDE94}. It builds a binary
search tree with the bins as the leaves and it does a random walk over
this tree. Due to lack of space the algorithm \IBR\ is presented in
Appendix~\ref{appendixA} but more intuition is given later in this
section.

\textbf{Ranking within each bin:} (Step $5$)
For each bin, we show that the number of elements far from both
anchors is bounded.  The algorithm checks elements inside a bin
whether they are close to any of the bin's anchors. For the elements that
are close to anchors, we rank them close to the anchor. And for the elements
that are away from both anchors we rank them using \RankX\  and
output the resulting ranking.

\begin{algorithm}[h]
\caption{\textsc{Binary-Search-Ranking}}
\textbf{Input:} Set $S$, bias $\epsilon$. \newline
\textbf{Initialize:} $\epsilon' = \epsilon/16$, $\epsilon'' = \epsilon/15$, and 
$S^{o}=\emptyset$. $S_j = \emptyset$, $C_j = \emptyset$ and $B_j= \emptyset$, for $1 \leq j \leq \left \lfloor \frac{n}{(\log n)^x} \right \rfloor
  +1$.
\begin{enumerate}
\item
\label{step:anchor_selection} Form a set $S'$ with $\left \lfloor\frac{n}{(\log n)^x}\right
\rfloor$ random elements from $S$. Remove these elements from $S$.
\item
\label{step:anchor_ranking} Rank $S'$ using \RankX\ $\Paren{S', \epsilon', \frac{1}{n^6}}$.
 \item
\label{step:dummy_elements} Add dummy element $a$ at the beginning of $S'$ such that $p(a,e) = 0$
$\forall e \in S \bigcup S'$. Add dummy element $b$ at the end of $S'$
such that $p(b,e) = 1$ $\forall e \in S \bigcup S'$.
\item \textbf{for} $e \in S$:
\label{step:binning}
\begin{enumerate}
\item \label{step:binning_element}k = \IBR$(S', e, \epsilon'')$.
\item Insert $e$ in $S_k$.
\end{enumerate}
\item \textbf{for} $j = 1$ to $\left\lfloor\frac{n}{(\log n)^x}\right\rfloor+1$:
\begin{enumerate}
\item \label{step:binning_further}\textbf{for} $e \in S_j$:
\begin{enumerate}
\item \label{step:check_close}\textbf{if} \textsc{Compare2}$(e, S'(j), {10\epsilon''^{-2} \log
  n}) \in \left[\frac12 - 6\epsilon'', \frac12 + 6\epsilon''
  \right]$ , insert $e$ in $C_j$.
\item \label{step:check_far}\textbf{else if} \textsc{Compare2}$(e,S'(j+1),10{\epsilon''^{-2} \log
  n}) \in \left[\frac12 - 6\epsilon'', \frac12 + 6\epsilon''
  \right]$, then insert $e$ in $C_{j+1}$.
\item \textbf{else} insert $e$ in $B_j$.
\end{enumerate}
\item \label{step:bin_ranking}Rank $B_j$ using \RankX\ $\Paren{B_j,\epsilon'',\frac{1}{n^4}}$.
\item Append $S'(j)$, $C_j$, $B_j$ in order at the end of $S^o$.
\end{enumerate}
\end{enumerate}
\textbf{Output:} $S^o$
\end{algorithm}

\begin{algorithm}
\caption{\textsc{Compare2}}
\textbf{Input:} element $a$, element $b$, number of comparisons $k$.
\begin{enumerate}
\item Compare $a$ and $b$ for $k$ times and return the fraction of times
  $a$ wins over $b$.
\end{enumerate}
\end{algorithm}

\subsubsection{Analysis of \textsc{Binary-Search-Ranking}}
\textbf{Creating anchors}
In Step $1$ of the algorithm we select $n/(\log n)^x$ random
elements. Since these are chosen uniformly random, they lie nearly
uniformly in the set $S$.  This intuition is formalized in the next
lemma.
\begin{Lemma}
\label{lem:bin_bound}
Consider a set $S$ of $n$ elements.  If we select $\frac{n}{(\log
  n)^x}$ elements uniformly randomly from $S$ and build an ordered set $S'$
s.t. $\tilde{p}(S'(i),S'(j)) \ge 0$ $\forall i > j$ , then with
probability $\ge 1- \frac1{n^4}$, for any $\epsilon > 0$ and all $k$,
\[
|\{e \in S:\tilde{p}(e, S'(k)) > \epsilon,
\tilde{p}( S'(k+1), e) > \epsilon\}|  \le 5(\log n)^{x+1}.
\]
\end{Lemma}
In Step ~\ref{step:anchor_ranking}, we use \RankX\  to rank $S'$.
Lemma~\ref{lem:anchor_ranking} shows the guarantee of ranking $S'$.
\begin{Lemma}
\label{lem:anchor_ranking}
After Step~\ref{step:anchor_ranking} of the
\textsc{Binary-Search-Ranking} with probability $\geq 1-
\frac{1}{n^6}$, $S'$ is $\epsilon'$-ranked.
\end{Lemma}
At the end of Step ~\ref{step:anchor_ranking}, we have $\frac{n}{(\log
n)^x}-1$ bins, each between two successively ranked anchors. Each bin
has a left anchor and a right anchor . We say that an element belongs
to a bin if it wins over the bin's left anchor with probability
$\ge \half$ and wins over the bin's right anchor with probability $\le
\half$. Notice that some elements might win over $S'(1)$ with probability
$< \half$ and thus not belong to any bin. So in
Step~\ref{step:dummy_elements}, we add a dummy element $a$ at the
beginning of $S'$ where $a$ loses to every element in $S\bigcup S'$ with
probability $1$. For similar reasons we add a dummy element $b$ to the end of
$S'$ where every element in $S \bigcup S'$ loses to $b$ with probability $1$.

\textbf{Coarse Ranking} 
Note that $S'(i)$ and $S'(i+1)$ are respectively
the left and right anchors of the bin $S_i$.

Since $S'$ is $\epsilon'$-ranked and the comparisons are noisy, it is
hard to find a bin $S_i$ for an element $e$ such that $p(e,
S'(i)) \ge \half$ and $p(S'(i+1),e) \ge \half$. We call a bin $S_i$
a $\epsilon''-$\emph{nearly correct} bin for an element $e$ if $p(e,
S'(i)) \ge \half 
\epsilon''$ and $p(S'(i+1), e) \ge \half -  \epsilon''$ for some
$\epsilon'' > \epsilon'$.
   
In Step~\ref{step:binning}, for each element we find a $\epsilon''$-\emph{nearly correct}
bin using \IBR\ . Next we describe an outline of \IBR.

\BFIBR\ first builds a binary search tree of intervals (see
Appendix~\ref{appendixA}) as follows: the root node is the entire
interval between the first and the last elements in $S'$. Each
non-leaf node interval $I$ has two children corresponding to the left
and right halves of $I$.  The leaves of the tree are the bins between
two successively ranked anchors.

To find a $\epsilon''$-\emph{nearly correct} bin for an element $e$, the algorithm starts at
the root of the binary search tree and at every non-leaf node
corresponding to interval $I$, it checks if $e$ belongs to $I$ or not
by comparing $e$ with $I$'s left and right anchors. If $e$ loses to
left anchor or wins against the right anchor, the algorithm
backtracks to current node's parent.

If $e$ wins against $I$'s left anchor and loses to its right one,
the algorithm checks if $e$ belongs to the left child or the right one
by comparing $e$ with the middle element of $I$ and moves accordingly.

When at a leaf node, the algorithm checks if $e$ belongs to the bin by
maintaining a counter. If $e$ wins against the bin's left anchor and
loses to the bin's right anchor, it increases the counter by one and
otherwise it decreases the counter by one. If the counter is less than
0 the algorithm backtracks to the bin's parent. By repeating each
comparison several times, the algorithm makes a correct decision with
probability $\ge \frac{19}{20}$.

Note that there could be several $\epsilon''$-\emph{nearly correct} bins for $e$ and even
though at each step the algorithm moves in the direction of one of
them, it could end up moving in a loop and never reaching one of
them. We thus run the algorithm for $30 \log n$ steps and terminate.

If the algorithm is at a leaf node by $30 \log n$ steps and the
counter is more than $10 \log n$ we show that the leaf node bin
is a $\epsilon''$-\emph{nearly correct} bin for $e$ and the algorithm outputs the leaf
node. If not, the algorithm puts in a set $Q$ all the anchors
visited so far and orders $Q$ according to $S'$. 

We select $30 \log n$ steps to ensure that if there is only 
one nearly correct bin, then the algorithm outputs that 
bin w.p. $\ge 1-\frac1{n^6}$. Also we do not want too many steps
so as to bound the size of $Q$.  

By doing a simple binary search in $Q$ using
\BR\ (see Appendix~\ref{appendixA}) we find an anchor $f
\in Q$ such that $|\tilde{p}(e,f)| \le 4
\epsilon''$.  Since \IBR\ ran for at most $30 \log n$ steps, $Q$ can
have at most $60 \log n$ elements and hence \BR\ can search
effectively by repeating each comparison $\cO(\log n)$ times to maintain
high confidence. Next paragraph explains how \BR\ finds such an element $f$. 

\BFBR\ first compares $e$ with the middle element $m$ of $Q$ for $\cO(\log
n)$ times. If the fraction of wins for $e$ is between $\half - 3
\epsilon''$ and $\half + 3\epsilon''$, then w.h.p. $|\tilde{p}(e,m)| \le 4 \epsilon''$ and hence \BR\
outputs $m$. If the fraction of wins for $e$ is less than $\half - 3
\epsilon''$, then w.h.p. $\tilde{p}(e,m) \le  -2 \epsilon''$ and hence
it eliminates all elements to the right of $m$ in $Q$. If the fraction of wins for
$e$ is more than $\half + 3 \epsilon''$, then w.h.p. $\tilde{p}(e,m) \ge
2 \epsilon''$ and hence it eliminates all elements to the left
of $m$ in $Q$.  It continues this process until it finds an element $f$ such that
the fraction of wins for $e$ is between $\half - 3 \epsilon''$ and 
$\half + 3 \epsilon''$.

In next Lemma, we show that \IBR\ achieves to 
find a $5\epsilon''$-\emph{nearly correct} bin for every element.
\begin{Lemma}
For any element $e \in S$, Step~\ref{step:binning} of
\textsc{Binary-Search-Ranking} places $e$ in bin $S_l$ such that
$\tilde{p}( e, S'(l)) > -5 \epsilon''$ and $\tilde{p}(S'(l+1), e) > -5
\epsilon''$ with probability $\ge 1- \frac1{n^5}$.
\label{lem:right_bin}
\end{Lemma}

\textbf{Ranking within each bin}
Once we have identified the bins, we rank the elements inside each
bin. By Lemma~\ref{lem:bin_bound}, inside each bin all elements are
close to the bin's anchors except at most $10 (\log n)^{x+1}$ of them.

The algorithm finds the elements close to anchors in Step $5a$ by
comparing each element in the bin with the bin's anchors. If
an element in bin $S_j$ is close to bin's anchors $S'(j)$ or
$S'(j+1)$ , the algorithm moves it to the set $C_j$ or $C_{j+1}$ accordingly
and if it is far away from both, the algorithm  moves it to the set
$B_j$. The following two lemmas state that this separating process happens
accurately with high probability. The proofs of these results follow
from the Chernoff bound and hence omitted.
\begin{Lemma} At the end of Step~\ref{step:binning_further}, for all $j$,
  $\forall e \in C_j$, $|\tilde{p}(e, S'(j))| <
  7\epsilon''$ with probability $\ge 1- \frac1{n^3}$.
\label{lem:boundary_bin}
\end{Lemma}
\begin{Lemma} 
At the end of Step~\ref{step:binning_further}, for all $j$, $\forall
e \in B_j$, $\min(\tilde{p}(e, S'(j)),\tilde{p}(S'(j+1), e)) >
5 \epsilon''$ with probability $\ge 1-\frac1{n^3}$.
\label{lem:middle_bin}
\end{Lemma}
Combining Lemmas~\ref{lem:bin_bound},~\ref{lem:anchor_ranking}
and~\ref{lem:middle_bin} next lemma shows that the size of $B_j$ is bounded for
all $j$.
\begin{Lemma}
At the end of Step~\ref{step:binning_further}, $|B_j| \le 10 (\log
n)^{x+1}$ for all $j$, with probability $\ge 1-\frac3{n^3}$.
\label{lem:real_bin_bound}
\end{Lemma}

Since all the elements in $C_j$ are already close to an anchor, they do
not need to be ranked. By Lemma~\ref{lem:bin_bound} with probability $\ge
1-\frac3{n^3}$ the number of elements in $B_j$ is at most
$10 (\log n)^{x+1}$. Therefore we use \RankX\  to rank these
elements and output the final ranking.

Lemma~\ref{lem:bin_ranking_accuracy} shows that all $B_j$'s are
$\epsilon''$-ranked at the end of Step~\ref{step:bin_ranking}.  Proof
follows from properties of \RankX\  and union bound.
\begin{Lemma}
\label{lem:bin_ranking_accuracy}
At the end of Step~\ref{step:bin_ranking}, all $B_j$s are
$\epsilon''$-ranked with probability $\ge 1-\frac1{n^3}$.
\end{Lemma}

Combining the above set of results yields our main result.
\begin{Theorem}
\label{thm:ranking_upper}
Given access to $\RankX\ $, 
\textsc{Binary-Search-Ranking} uses $O\Paren{\frac{n \log n (\log \log
    n)^x}{\epsilon^2}}$ comparisons and produces an $\epsilon$-ranking
with probability $\ge 1- \frac{1}{n}$.
\end{Theorem}
Using \textsc{Merge-Rank}
as a \RankX\  algorithm with $x=3$ leads to the following corollary. 
\begin{Corollary}
\label{cor}
\textsc{Binary-Search-Ranking} uses $O\Paren{\frac{n \log n (\log \log n)^3}
{\epsilon^2}}$ comparisons and produces an $\epsilon$-ranking
with probability $\ge 1- \frac{1}{n}$.
\end{Corollary}
Using \textsc{PALPAC-AMPRR}~\cite{BalazsRAE15} as a 
\RankX\ algorithm with $x=1$ leads to the following 
corollary over PL model.
\begin{Corollary}
\label{cor:2}
Over PL model, \textsc{Binary-Search-Ranking} uses $O\Paren{\frac{n \log n \log \log n}                                                      
{\epsilon^2}}$ comparisons and produces an $\epsilon$-ranking
with probability $\ge 1- \frac{1}{n}$.
\end{Corollary}

It is well known that to rank a set of $n$ values under the noiseless
setting, $\Omega(n \log n)$ comparisons are necessary. We show that
under the noisy model, $\Omega\Paren{\frac{n}{\epsilon^2} \log
  \frac{n}{\delta}}$ samples are necessary to output an
$\epsilon$-ranking and hence our algorithm is near-optimal.

\begin{Theorem}
There exists a noisy model that satisfies strong stochastic
transitivity and stochastic triangle inequality such that to output an
$\epsilon$-ranking with probability $\geq 1-\delta$,
$\Omega\Paren{\frac{n}{\epsilon^2} \log \frac{n}{\delta}}$ comparisons
are necessary.
\label{thm:ranking_lower}
\end{Theorem}

\section{Experiments}
\label{sec:experiments}
We compare the performance of our algorithms with that of others over
simulated data. Similar to~\cite{YisongT11}, we consider the
stochastic model where $p(i,j) = 0.6\ \forall i < j$. Note that this
model satisfies both strong stochastic transitivity and triangle
inequality. We find $0.05$-maximum with error probability $\delta =
0.1$. Observe that $i=1$ is the only $0.05$-maximum. We compare the
sample complexity of \textsc{Knockout} with that of
\textbf{BTM-PAC}~\cite{YisongT11}, \textbf{MallowsMPI}
~\cite{RobertEB14}, and
\textbf{AR}~\cite{HeckelSRW16}. \textbf{BTM-PAC} is an $(\epsilon,
\delta)$-PAC algorithm for the same model considered in this
paper. \textbf{MallowsMPI} finds a Condorcet winner which exists under
our general model. \textbf{AR} finds the maximum according to
\emph{Borda} scores.  We also tried \textbf{PLPAC}~\cite{BalazsRAE15},
developed originally for PL model but the algorithm could not meet
guarantees of $\delta=0.1$ under this model and hence omitted.  Note
that in all the experiments the reported numbers are averaged over 100
runs.
\begin{figure}[H]
\centering
\includegraphics[scale=0.4]{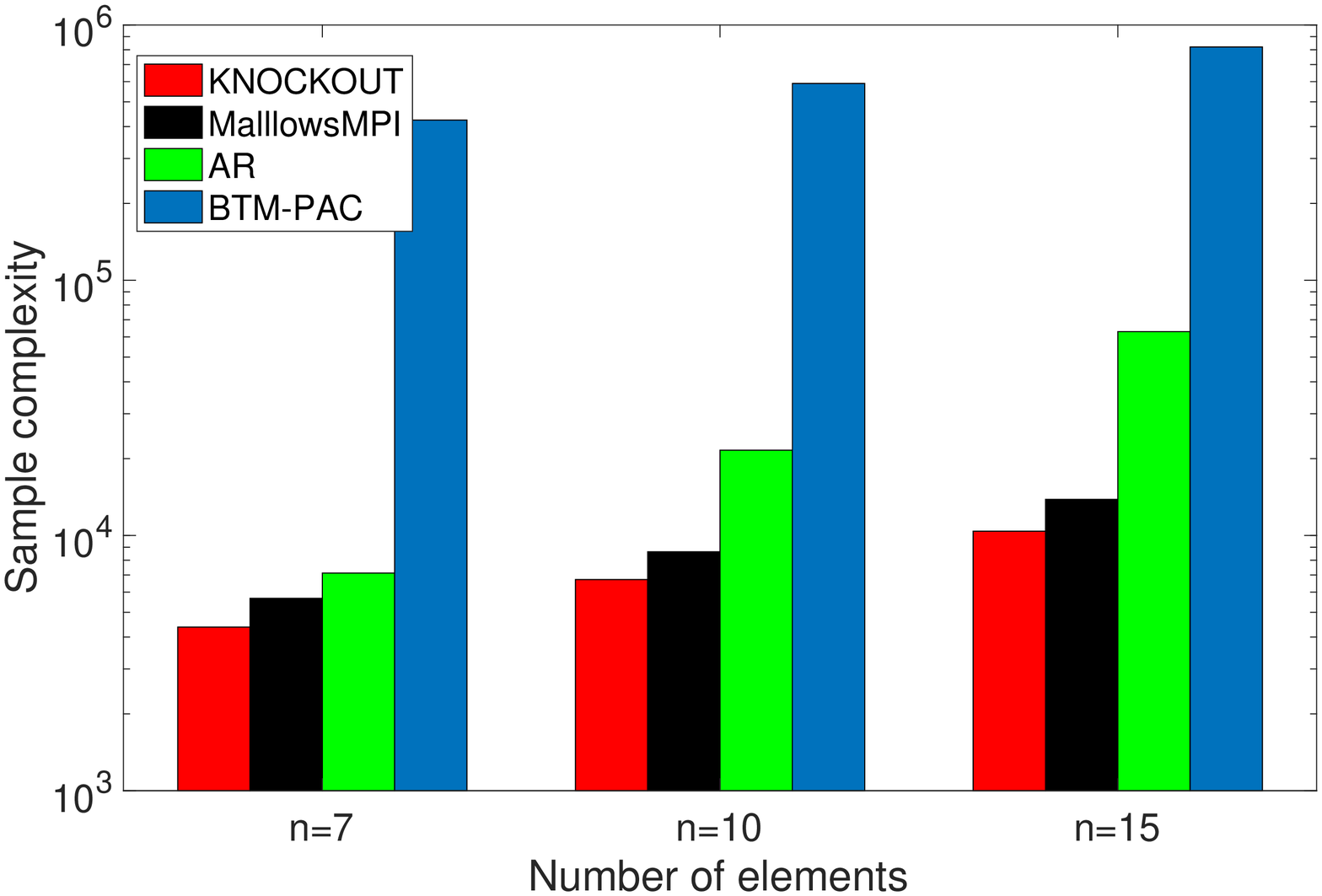}
\caption{Comparison of sample complexity for several models, and different input sizes
, with $\epsilon = 0.05$ and $\delta = 0.1$}
\label{fig:1}
\end{figure}

In Figure~\ref{fig:1}, we compare the sample complexity of algorithms
when there are 7, 10 and 15 elements. Our algorithm outperforms all
the others. \textbf{BTM-PAC} performs much worse in comparison to
others because of high constants in the algorithm. Further \textbf{BTM-PAC}
allows comparing an element with itself since the main objective in 
~\cite{YisongT11} is to reduce the regret. We include more comparisons 
with \textbf{BTM-PAC} in Appendix~\ref{sec:app_exp}. We exclude \textbf{BTM-PAC} for further experiments with higher
number of elements.
\begin{figure}[H]
\centering
\includegraphics[scale=0.4]{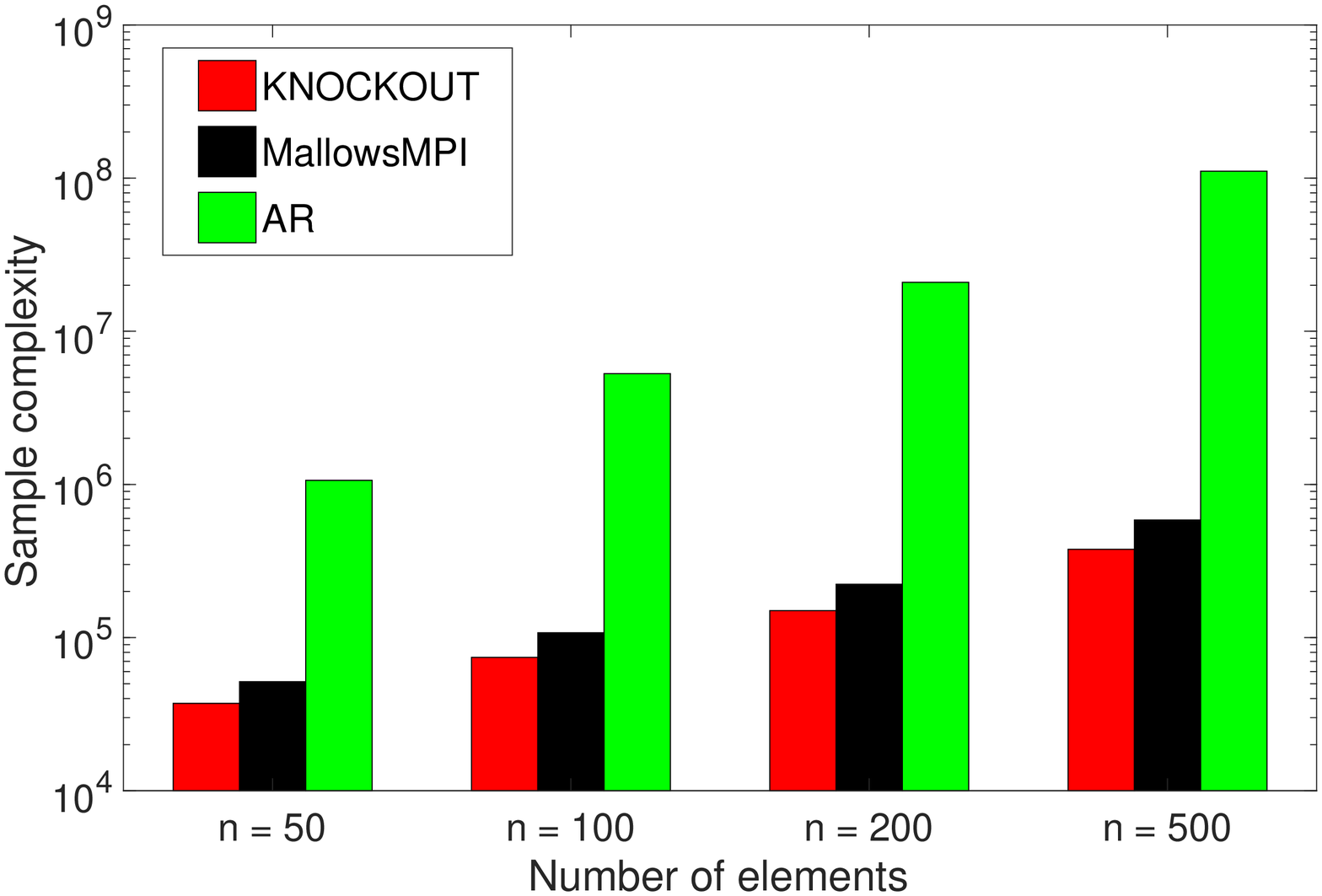}
\caption{Comparison of sample complexity for large input size, $\epsilon=0.05$, and 
$\delta=0.1$}
\label{fig:2}
\end{figure}

In Figure~\ref{fig:2}, we compare the algorithms when there are 50,
100, 200 and 500 elements.  Our algorithm outperforms others for
higher number of elements too. Performance of \textbf{AR} gets worse
as the number of elements increases since Borda scores of the elements
get closer to each other and hence \textbf{AR} takes more comparisons
to eliminate an element. Notice that number of comparisons is in
logarithmic scale and hence the performance of \textbf{MallowsMPI}
appears to be close to that of ours.

As noted in~\cite{BalazsRAE15}, sample complexity of \textbf{MallowsMPI} gets
worse as $\tilde{p}(i,j)$ gets close to $0$. To show the pronounced
effect, we use the stochastic model $p(1,j) = 0.6\ \forall j>1$,
$p(i,j) = 0.5 +\tilde{p}\ \forall j > i, i>1$ where $\tilde{p} < 0.1$, 
and the number of elements is 15.  Here too we find $0.05$-maximum with $\delta=0.1$. Note
that $i=1$ is the only $0.05$-maximum in this stochastic model.
\begin{figure}[H]
\centering
\includegraphics[scale=0.4]{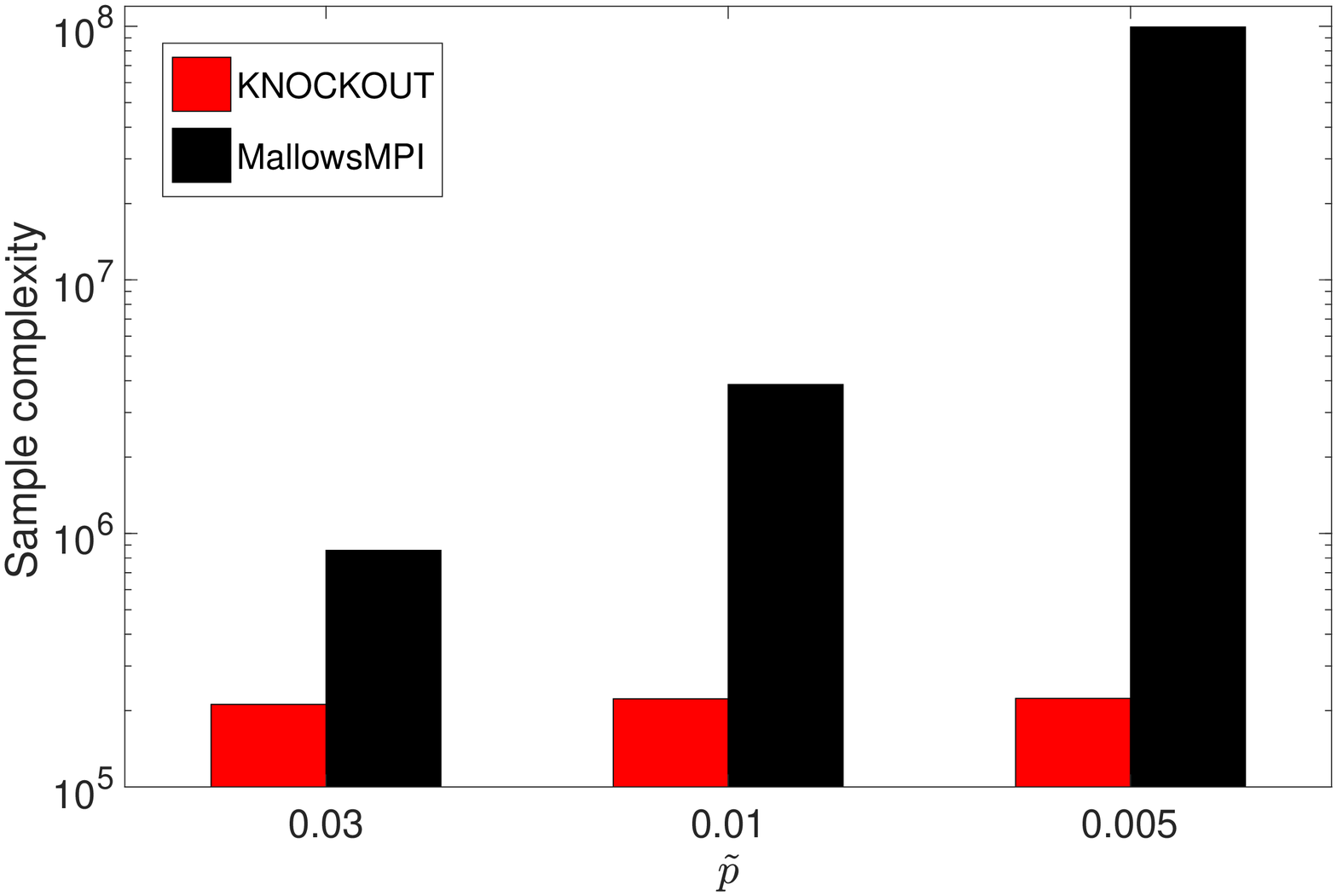}
\caption{Sample complexity comparison of \textsc{Knockout} with MallowsMPI
for different values of $\tilde{p}$, with $\epsilon=0.05$ and $\delta = 0.1$}
\label{fig:3}
\end{figure}
In Figure~\ref{fig:3}, we compare the algorithms for different values
of $\tilde{p}$: 0.01, 0.005 and 0.001. As discussed above, the
performance of \textbf{MallowsMPI} gets much worse whereas our
algorithm's performance stays unchanged. The reason is that
\textbf{MallowsMPI} finds the Condorcet winner using successive
elimination technique and as $\tilde{p}$ gets closer to 0,
\textbf{MallowsMPI} takes more comparisons for each elimination. Our
algorithm tries to find an alternative which defeats Condorcet winner
with probability $\ge 0.5-0.05$ and hence for alternatives that are
very close to each other, our algorithm declares either one of them as
winner after comparing them for certain number of times.

\begin{figure}[H]
\centering
\includegraphics[scale=0.4]{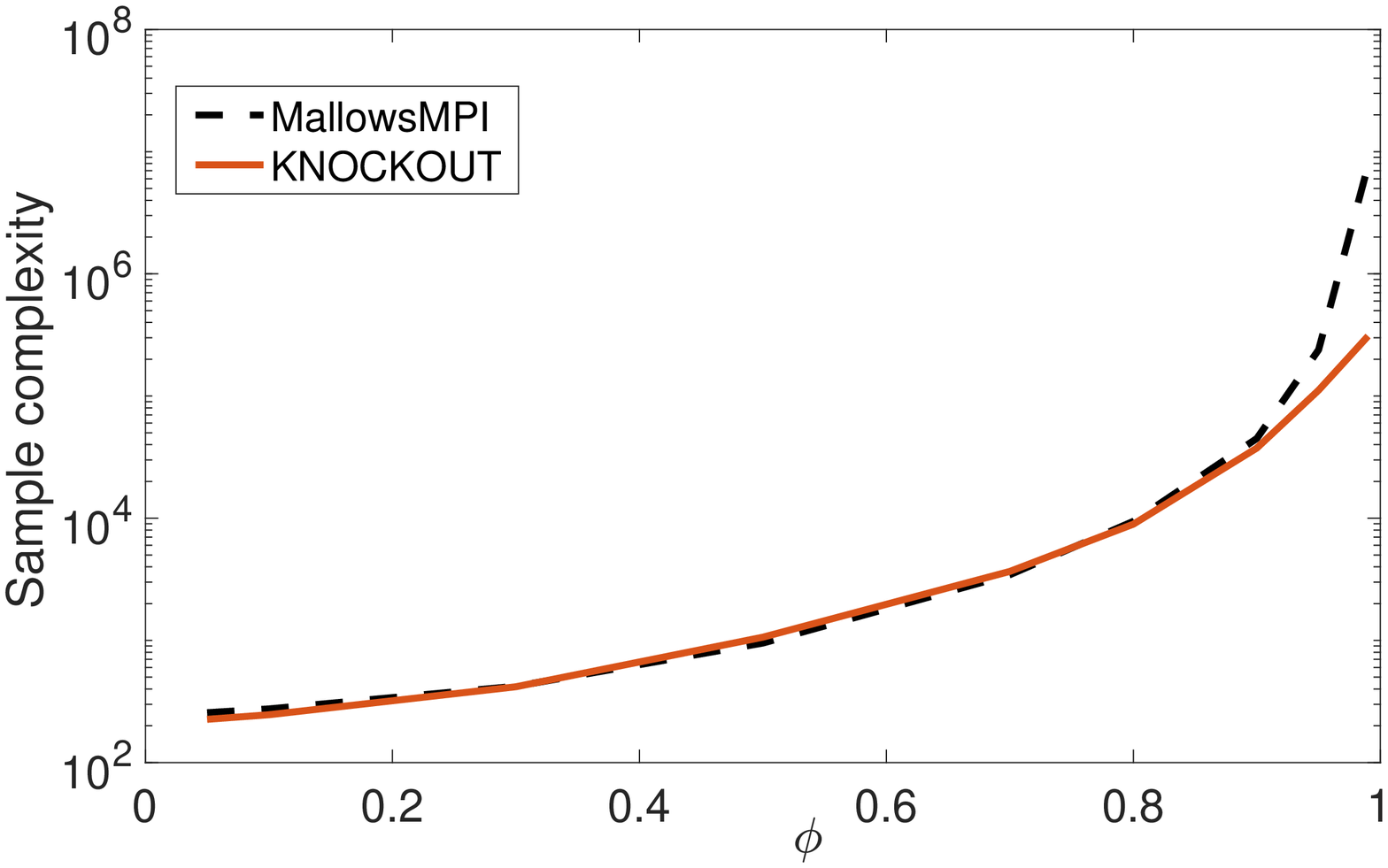}
\caption{Sample complexity of \textsc{Knockout} and \textbf{MallowsMPI} under Mallows model for various values of $\phi$}
\label{fig:6}
\end{figure}
Next we evaluate \textsc{Knockout} on Mallows model which does not satisfy 
triangle inequality. Mallows is a parametric model which is specified by 
single parameter $\phi$. As in~\cite{RobertEB14}, we consider $n=10$ elements 
and various values for $\phi$: 0.03, 0.1, 0.3, 0.5, 0.7, 0.8, 0.9, 0.95 and 
0.99. Here again we seek to find $0.05$-maximum with $\delta = 0.05$. 
As we can see in Figure~\ref{fig:6}, sample complexity of \textsc{Knockout} and \textbf{MallowsMPI} 
is essentially same under small values of $\phi$ but 
\textsc{Knockout} outperforms \textbf{MallowsMPI} as $\phi$ gets close to $1$ since 
comparison probabilities grow closer to $1$. Surprisingly, for all values of 
$\phi$ except for 0.99, \textsc{Knockout} returned Condorcet winner in all runs.
For $\phi = 0.99$, \textsc{Knockout} returned second best element in 10 runs out of 100.
Note that $\tilde{p}(1,2) = 0.0025$ and hence \textsc{Knockout} still outputed a 
$0.05$-maximum. Even though we could not show theoretical guarantees of \textsc{Knockout}
under Mallows model, our simulations suggest that it can perform well even under this model.   

More experiments are provided in Appendix~\ref{sec:app_exp}.

\section{Conclusion}
\label{sec:conclusion}
We studied maximum selection and ranking using noisy comparisons for
the broad model where the comparison probabilities satisfy strong
stochastic transitivity and the triangle inequality.  For maximum
selection, we presented a simple algorithm with linear, hence optimal,
sample complexity.  For ranking we presented a framework that improves
the performance of many ranking algorithms and applied it to merge
ranking to derive a near-optimal ranking algorithm.

We conducted several experiments and showed that our
algorithms perform well not only in theory, but also in practice.
Furthermore, they out-performed all existing algorithms. 

The maximum-selection experiments suggest that our algorithm
performs well even without the triangle-inequality assumption.
It would be of interest to extend our theoretical guarantees to
this case. 
For ranking, it would be interesting to close the
$(\log\log n)^3$ ratio between the upper- and lower- complexity bounds. 

\newpage
\bibliographystyle{plainnat}
\bibliography{main}
\newpage
\appendix
\section{Merge Ranking}
\label{app:merge}
We first introduce a subroutine that is used by \textsc{Merge-Rank}.
It merges two ordered sets in the presence of noisy comparisons.
\subsection{\textsc{Merge}}
\textsc{Merge} takes two ordered sets $S_1$ and $S_2$ and outputs an
ordered set $Q$ by merging them.  \textsc{Merge} starts by comparing
the first elements in each set $S_1$ and $S_2$ and places the loser in
the first position of $Q$. It compares the two elements sufficient
times to make sure that output is near-accurate. Then it compares the
winner and the element right to loser in the corresponding set.  It
continues this process until we run out of one of the sets and then
adds the remaining elements to the end of $Q$ and outputs $Q$.
\begin{algorithm}
\caption{\textsc{Merge}}
\textbf{Input:} Sets $S_1, S_2$, bias $\epsilon$, confidence $\delta$.
\\\textbf{Initialize:} $i=1$, $j=1$ and $O = \emptyset$.
\begin{enumerate}
\item \textbf{while} $i \le |S_1|$ and $j \le |S_2|$.
\begin{enumerate}
\item \textbf{if} $S_1(i) = \textsc{Compare}\Paren{S_1(i),
S_2(j), \epsilon, \delta}
$, then append $S_1(i)$ at the end of $O$ and $i=i+1$.
\item \textbf{else} append $S_2(j)$ at the end of $O$ and $j=j+1$.
\end{enumerate}
\item \textbf{if} $i \le |S_1|$, then append $S_1(i : |S_1|)$ at the end of $O$.
\item \textbf{if} $j \le |S_2|$, then append $S_2(j:|S_2|)$ at the end of $O$.
\end{enumerate}
\textbf{Output:} $O$.
\end{algorithm}

We show that when we merge two ordered sets using \textsc{Merge},
the error of resulting ordered set is not high compared to the maximum
of errors of individual ordered sets.
\begin{Lemma}
\label{lem:error_add_bound}
With probability $\ge 1- (|S_1|+ |S_2|)\delta$, error of
$\textsc{Merge}\Paren{S_1,S_2, \epsilon, \delta}$ is at most
$\epsilon$ more than the maximum of errors of $S_1$ and $S_2$. Namely,
with probability $\ge 1-(|S_1| + |S_2|)\delta$,
\[
err\Paren{\textsc{Merge}\Paren{S_1, S_2, \epsilon, \delta}}
\le \max\Paren{err(S_1), err(S_2)} + \epsilon.
\]
\end{Lemma}
\subsection{\textsc{Merge-Rank}}
Now we present the algorithm \textsc{Merge-Rank}. \textsc{Merge-Rank}
partitions the input set $S$ into two sets $S_1$ and $S_2$ each of
size $|S|/2$.  It then orders $S_1$ and $S_2$ separately using
\textsc{Merge-Rank} and combines the ordered sets using
\textsc{Merge}. Notice that \textsc{Merge-Rank} is
a recursive algorithm. The singleton sets each containing an unique
element in $S$ are merged first. Two singleton sets are merged to form
a set with two elements, then the sets with two elements are merged to
form a set with four elements and henceforth.  By
Lemma~\ref{lem:error_add_bound}, each merge with bound parameter
$\epsilon'$ adds at most $\epsilon'$ to the error. Since error of
singleton sets is $0$ and each element takes part in $\log n$ merges,
the error of the output set is at most $\epsilon'\log n$.  Hence with
bound parameter $\epsilon/ \log n$, the error of the output set is
less than $\epsilon$.
\begin{algorithm}
\caption{\textsc{Merge-Rank}}
\textbf{Input:} Set $S$, bias $\epsilon$, confidence $\delta$.
\begin{enumerate}
\item $  S_1 = \textsc{Merge-Rank}\Paren{S(1:\floor{|S|/2}),  \epsilon, \delta}$.
\item $S_2 = \textsc{Merge-Rank}\Paren{S(\floor{|S|/2}+1:|S|),  \epsilon, \delta}$.
\end{enumerate}
\textbf{Output:} $\textsc{Merge}\Paren{S_1, S_2,  \epsilon, \delta}$.
\end{algorithm}
\vspace{12ex}

\section{Algorithms for Ranking}
\label{appendixA}
\vspace{10ex}
\begin{algorithm}[H]
\caption{\textsc{Interval-Binary-Search}}
\textbf{Input:} Ordered array $S$, search element $e$, bias $\epsilon$
\begin{enumerate}
\item $T$ = \textsc{Build-Binary-Search-Tree}$(|S|)$.
\item Initialize set $Q = \emptyset$, node $\alpha = \text{root}(T)$, and count $c =                                                          
  0$.
\item \textbf{repeat} for $30 \log n$ times
\begin{enumerate}
\item \textbf{if} $\alpha_2 - \alpha_1 > 1$,
\begin{enumerate}
\item Add $\alpha_1$, $\alpha_2$ and $\left\lceil\frac{\alpha_1+\alpha_2}2\right\rceil$ to $Q$.
\item \textbf{if} \textsc{Compare}$(S(\alpha_1),e, \frac{10}{\epsilon^{2}}) >
  1/2$ or \textsc{Compare}$(e, S(\alpha_2), \frac{10}{\epsilon^{2}}) >
  1/2$ then go back to the parent, $\alpha = \text{parent}(\alpha).$
\item
\textbf{else}
\begin{itemize}
\item \textbf{if}
  \textsc{Compare}$(S(\left\lceil\frac{\alpha_1+\alpha_2}2\right\rceil),e,
  \frac{10}{\epsilon^{2}}) > 1/2$ go to the left child,$\alpha = \text{left}(\alpha).$
\item
\textbf{else} go to the right child, $\alpha = \text{right}(\alpha).$
\end{itemize}
\end{enumerate}
\item \textbf{else}
\begin{enumerate}
\item \textbf{if} \textsc{Compare}$(e, S(\alpha_1), \frac{10}{\epsilon^{2}}) > 1/2$
  and \textsc{Compare}$(S(\alpha_2),a,\frac{10}{\epsilon^{2}}) > 1/2$,
\[
c = c+1.
\]
\item \textbf{else}
\begin{enumerate} 
\item \textbf{if} $c = 0$, $\alpha = \text{parent}(\alpha)$. 
\item \textbf{else}  $c = c-1$.
\end{enumerate}
\end{enumerate}
\end{enumerate}
\item
\begin{enumerate}
\item \textbf{if} $c > 10\log n$, \textbf{Output:} $\alpha_1$.
\label{output}
\item \textbf{else} \textbf{Output:} \BR$(S, Q, e, 2\epsilon)$.
\end{enumerate}
\end{enumerate}
\end{algorithm}
\vspace{25ex}

\begin{algorithm}[H]
\caption{\textsc{Build-Binary                                                                                                     
      -Search-Tree}} \textbf{Input:} size $n$.
\newline // Recall that
    each node $m$ in the tree is an interval between left end $m_1$ and
    right end $m_2$.
\begin{enumerate}
\item Initialize set $T' =\emptyset$.
\item Initialize the tree $T$ with the root node $(1,n)$.
\begin{align*}
m = (1,n) \qquad &\text{where }m_1 = 1 \text{ and } m_2 = n,\\
\text{root}(T) &= m
\end{align*}
\item Add $m$ to $T'$.
\item \textbf{while} $T'$ is not empty
\begin{enumerate}
\item Consider a node $i$ in $T'$.
\item \textbf{if} $i_2 - i_1 > 1$,
create a left child and right child to $i$ and set their parents as $i$.
\begin{align*}
\alpha = \Paren{i_1, \left\lceil\frac{i_1 + i_2}2\right\rceil},
\qquad&\beta = \Paren{\left\lceil{\frac{i_1 + i_2}2}\right\rceil, i_2}, \\
\text{left}(i) = \alpha, \qquad &\text{right}(i) = \beta,\\
\text{parent}(\alpha) = i,\qquad  & \text{parent}(\beta) = i.
\end{align*}
and add nodes $\alpha$ and $\beta$ to $T'$.
\item Remove node $i$ from $T'$.
\end{enumerate}
\end{enumerate}
\textbf{Output:} $T$.
\end{algorithm}

\begin{algorithm}[H]
\caption{\BR} \textbf{Input:} Ordered array $S$, ordered array $Q$,
search item $e$, bias $\epsilon$.\\
\textbf{Initialize:} $l =1$, $h = |Q|$.
\begin{enumerate}
\item \textbf{while} $h-l >0$
\begin{enumerate}
\item $t =$ \textsc{Comapre}$\Paren{e,
  S(Q(\left\lceil\frac{l+h}2\right\rceil), {\frac{10\log
      n}{\epsilon^2}}}$.
\item \textbf{if} $t \in \left[\frac12 -3\epsilon,\frac12 + 3\epsilon \right]$,
  then \textbf{Output:} $Q(\left\lceil\frac{l+h}2\right\rceil)$.
\item \textbf{else if} $t < \frac12 -3\epsilon$, then move to the right.
\[
l = \left\lceil\frac{l+h}2\right\rceil.
\]
\item \textbf{else} move to the left.
\[
h =  \left\lceil\frac{l+h}2\right\rceil .
\]
\end{enumerate}
\end{enumerate}
\textbf{Output:} $Q(h)$.
\end{algorithm}

\section{Some tools for proving lemmas}
\label{sec:tools}
We first prove an auxilliary result that we use in the future analysis.

\begin{Lemma}
Let $W=\textsc{Compare}(i,j,\epsilon,\delta)$ and $L$ be the other element. Then with
probability $\ge 1-\delta$,
\[
p({W,L}) \geq \frac{1}{2} - \epsilon.
\]
\label{lem:chernoff}
\end{Lemma}
\begin{proof}
Note that if $|\tilde{p}(i,j)| < \epsilon$, then $p(i,j) > \frac12
-\epsilon$ and $p(j,i) > \frac12 - \epsilon$. Hence, $p(W,L) \geq
\frac12-\epsilon$.

If $|\tilde{p}(i,j)| \geq \epsilon$, without loss of generality,
assume that $i$ is a better element i.e., $\tilde{p}(i,j) \geq
\epsilon$. By Lemma \ref{lem:chernoff1}, with probability atleast
$1-\delta$, $W=i$. Hence
\[
Pr\Paren{p(W,L) \geq \frac12-\epsilon} = Pr(W=i) \geq 1-\delta.\qedhere
\]
\end{proof}

We now prove a Lemma that follows from strong stochastic transitivity and
stochastic triangle inequality that we will use in future analysis.
\begin{Lemma}
\label{lem:trans_tri}
If $\tilde{p}(i,j) \le \epsilon_1$, $\tilde{p}(j,k) \le \epsilon_2$,
then $\tilde{p}(i,k) \le \epsilon_1+ \epsilon_2$.
\end{Lemma}
\begin{proof}
We will divide the proof into four cases based on whether
$\tilde{p}(i,j) >0$ and $\tilde{p}(j,k) > 0$.

If $\tilde{p}(i,j) \le 0$ and $\tilde{p}(j,k) \le 0$, then by
strong stochastic transitivity, $\tilde{p}(i,k) \le 0 \le \epsilon_1+
\epsilon_2$.

If $0 < \tilde{p}(i,j) \le \epsilon_1$ and $0 < \tilde{p}(j,k) \le
\epsilon_2$, then by stochastic traingle inequality, $\tilde{p}(i,k)
\le \epsilon_1 + \epsilon_2$.

If $\tilde{p}(i,j) <0$ and $0 < \tilde{p}(j,k) \le \epsilon_2$, then
by strong stochastic transitivity, $\tilde{p}(i,k) \le \epsilon_2 \le
\epsilon_1+ \epsilon_2$.

If $0 < \tilde{p}(i,j) \le \epsilon_1$ and $\tilde{p}(j,k) < 0$, then
by strong stochastic transitivity, $\tilde{p}(i,k) \le \epsilon_1 \le
\epsilon_1+ \epsilon_2$.
\end{proof}

\section{Proofs of Section~\ref{sec:max}}

\subsection*{Proof of Lemma ~\ref{lem:chernoff1}}
\begin{proof}
Let $\hat{p}_i^r$ and $\hat{c}^r$ denote $\hat{p}_i$ and $\hat{c}$
respectively after $r$ number of comparisons.  Output of
$\textsc{Compare}(i,j,\epsilon,\delta)$ will not be $i$ only if
$\hat{p}_i^r < \frac12 + \epsilon - \hat{c}^r$ for any
$r<m=\frac{1}{2\epsilon^2}\log\frac{2}{\delta}$ or if $\hat{p_i} < 
\frac12$ for $r = m$. We will show that the probability of each of these
events happening is bounded by $\frac{\delta}{2}$. Hence by union bound, Lemma
follows.

After $r$ comparisons, by Chernoff bound,
\[
Pr(\hat{p}_i^r < \frac12 + \epsilon - \hat{c}^r) \le e^{-2r(\hat{c}^r)^2} =
e^{-\log \frac{4r^2}{\delta}} = \frac{\delta}{4r^2}.
\]
Using union bound,
\[
Pr(\exists r \text{ s.t. }\hat{p}_i^r \le \frac12+\epsilon - \hat{c}^r) \le \frac{\delta}{2}
\]
After $m=\frac1{2\epsilon^2}\log\frac{2}{\delta}$ rounds, by Chernoff bound,
\[
Pr(\hat{p}_i^m < \frac12) \le e^{-2m\epsilon^2} = \frac{\delta}{2}.\qedhere
\]
\end{proof}

\subsection*{Proof of Lemma~\ref{lem:knockout_round}}
\begin{proof}
Each of the $\frac{|S|}2$ pairs is compared  at most $\frac{1}{2\epsilon^2}\log                                                            
\frac2{\delta}$ times, hence the total comparisons is
$\le\frac{|S|}{4\epsilon^2}\log \frac2{\delta}$. Let $k^* =                                                                          
\text{max}(\textsc{Knockout-Round}(S, \epsilon, \delta))$ and $s^* =                                                              
\text{max}(S)$. Let $a$ be the element paired with $s^*$. There are
two cases: $\tilde{p}(s^*,a) \ge \epsilon$ and $\tilde{p}(s^*,a) <                                                                
\epsilon$.

If $\tilde{p}(s^*,a) \ge \epsilon$, by Lemma \ref{lem:chernoff1} with
probability $\ge1-\delta$, $s^*$ will win and hence by definitions of
$s^*$ and $k^*$, $\tilde{p}(s^*, k^*) = 0 \le \gamma \epsilon$.
Alternatively, if $\tilde{p}(s^*,a) < \epsilon$, let $\text{winner}(i,j)$ denote
the winner between $i$ and $j$ when compared for
$\frac{1}{2\epsilon^2} \log \frac1{\delta}$ times. Then,
\[
r(a) \stackrel{(a)}\le r(\text{winner}(s^*,a)) \stackrel{(b)}\le r(k^*)
\stackrel{(c)}\le r(s^*)
\]
where (a) follows from $r(a) \le r(s^*)$, (b) and (c) follow from the
definitions of $s^*$ and $k^*$ respectively.  From strong stochastic
tranisitivity on $a$, $k^*$ and $s^*$, $\tilde{p}(s^*, k^*) \le \gamma
\tilde{p}(s^*, a) \le \gamma \epsilon$.
\end{proof}

\subsection*{Proof of Theorem~\ref{thm:max}}
\begin{proof}
We first show that with probability $\ge 1 - \delta$, the output of
\textsc{Knockout} is an $\epsilon$-maximum. Let $\epsilon_i =                                                                     
c\epsilon/2^{i/3}$ and $\delta_i = \delta/2^i$. Note that
$\epsilon_i$ and $\delta_i$ are bias and confidence values
used in round $i$.
 Let $b_i$ be a maximum element in the set $S$ before round $i$. Then by Lemma \ref{lem:knockout_round},
with probability $\ge1-\delta_i$,
\begin{align}
\tilde{p}(b_i, b_{i+1}) \leq \frac{c \epsilon}{2^{i/3}}.
\label{eq: knockout_step}
\end{align}

\ignore{
. Let $a_{i}$ be the element that $b_i$
be paired with in round $i$. If $b_i$ wins over $a_i$, then $p(b_i,                                                               
b_{i+1}) \geq \frac{1}{2}$. If $b_i$ loses to $a_i$, then by
Lemma~\ref{lem:chernoff}, relaxed stochastic transitivity, and the
definition of $b_{i+1}$, with probability $\ge1-\delta_i$,
\[
d(b_i, b_{i+1}) \leq \gamma d(b_i, a_{i}) \leq \gamma\epsilon_i.
\]
}

By union bound, the probability that Equation~\ref{eq: knockout_step}  does not hold for some round $1\le i\le \log |S|$ is
\[
 \le \sum_{i=1}^{\log |S|} \delta_i = \sum_{i=1}^{\log |S|} \frac{\delta}{2^i} \le \delta.
\]

With probability $\ge 1- \delta$, Equation~\ref{eq: knockout_step} holds for all $i$ and by stochastic triangle inequality,
\[
\tilde{p}(b_1, b_{ \log{|S|} + 1 }) \leq \sum_{i=1}^{{\log |S|}}
\tilde{p}(b_i, b_{i+1}) \leq \sum^\infty_{i=1}
\frac{c\epsilon}{2^{i/3}} = \epsilon.
\]
We now bound the number of comparisons.  Let $n_i =                                                                               
\frac{|S|}{2^{i-1}}$ be the number of elements in the set at the
beginning of round $i$.  The number of comparisons at round $i$ is
\[
\le \frac{n_i}{2} \cdot \frac{\gamma^22^{2i/3}}{2c^2\epsilon^2} \cdot \log \frac{2^{i+1}}{\delta}.
\]
Hence the number of comparisons in all rounds is
\begin{align*}
\sum^{ \log |S|}_{i=1}\frac{|S|}{2^{i}} \cdot
\frac{\gamma^22^{2i/3}}{2c^2\epsilon^2} \cdot \log \frac{2^{i+1}}{\delta}
& \leq \frac{|S|\gamma^2}{ 2c^2\epsilon^2} \sum^{\infty}_{i=1}
\frac{1}{2^{i/3}} \Paren{i+ \log \frac2{\delta}} \\ & =
\frac{|S|\gamma^2}{ 2c^2\epsilon^2} \left(\frac{2^{1/3}}{c^2} +
\frac{1}{c} \log \frac2{\delta} \right) \\
& =  \cO\Paren{\frac{|S|\gamma^2}{\epsilon^2}\log\frac{1}{\delta}}.\qedhere
\end{align*}
\end{proof}

\section{Proofs of Section~\ref{sec:merge}}
\subsection*{Proof of Lemma~\ref{lem:error_add_bound}}
\begin{proof}
Let $Q = \textsc{Merge}\Paren{S_1,S_2, \epsilon, \delta}$.  We will
show that for every $k$, w.p. $\ge 1- \delta$, $\tilde{p}(Q(k), Q(l))
\le \max(err(S_1), err(S_2)) + \epsilon$ $\forall l > k$.  Note that
if this property is true for every element then $err(Q) \le
\max(err(S_1), err(S_2)) + \epsilon$. Since there are $|S_1|+ |S_2|$
elements in the final merged set, the Lemma follows by union bound.

If $S_1(i)$ and $S_2(j)$ are compared in \textsc{Merge} algorithm,
without loss of generality, assume that $S_1(i)$ loses i.e., $S_1(i)$
appears before $S_2(j)$ in $T$.  The elements that appear to the right
of $S_1(i)$ in $Q$ belong to set $Q_{\ge S_1(i)} = \{S_1(k): k> i\}
\bigcup \{S_2(k) : k \ge j\}$. We will show that w.p. $\ge 1- \delta$,
$\forall e \in Q_{\ge S_1(i)}$, $\tilde{p}(S_1(i),e) \le
\max\Paren{err(S_1), err(S_2)} + \epsilon $.

By definition of error of an ordered set,
\begin{align}
 \tilde{p}(S_1(i), S_1(k)) &\le err(S_1) \quad \forall k > i\label{eq:set1}\\
 \tilde{p}(S_2(j), S_2(k)) &\le err(S_2) \quad \forall k \ge j.\label{eq:set2}
\end{align}
By Lemma~\ref{lem:chernoff}, w.p. $\ge 1- \delta$,
\begin{align}
\tilde{p}(S_1(i),S_2(j)) \le \epsilon. \label{eq:comp}
\end{align}
Hence by Equations~\ref{eq:set2},~\ref{eq:comp} and Lemma
~\ref{lem:trans_tri}, w.p. $\ge 1- \delta$, $\tilde{p}(S_1(i), S_2(k))
\le \epsilon + err(S_2)$ $\forall k \ge j$.
\end{proof}

\subsection*{Proof of Lemma~\ref{lem:merge_rank}}

\begin{proof}
We first bound the total comparisons. Let $C(Q,\epsilon',\delta')$ be
the number of comparisons that the \textsc{Merge-Rank} uses on a set
$Q$. Since \textsc{Merge-Rank} is a recursive algorithm, \ignore{
  Number of comparisons used by \textsc{Merge-Rank} depends only on
  the number of items in the set, bound $\epsilon$ and error $\delta$.
}
\begin{align*}
C(Q,\epsilon',\delta') \le&
C(Q[1:\floor{|Q|/2}],\epsilon',\delta')\\ &+
C(Q[\floor{|Q|/2}:|Q|],\epsilon',\delta') + {\frac{|Q|}{2\epsilon'^2}
  \log \frac{2}{\delta'}}.
\end{align*}

From this one can obtain that $C(S, \epsilon',\delta') =
\cO\Paren{\frac{|S| \log |S|}{\epsilon'^2} \log
  \frac{1}{\delta'}}$. Hence,
\[
C\Paren{|S|, \frac{\epsilon}{ \log |S|}, \frac{\delta}{|S|^2}} =
\cO\Paren{\frac{|S| \log^3 |S|}{\epsilon^2} \log
  \frac{|S|^2}{\delta}}.
\]

Now we bound the error. By Lemma \ref{lem:error_add_bound}, with
probability $\ge 1- |Q|\delta$,
\begin{align}
\label{eq:merge_bound}
&err(\textsc{Merge-Rank}(Q,\epsilon',\delta')) \le \nonumber\\
&\max\{err\Paren{\textsc{Merge-Rank}\Paren{Q[1:\floor{|Q|/2}], \epsilon', \delta'}},\nonumber \\
&err\Paren{\textsc{Merge-Rank}\Paren{T[\floor{|Q|/2}+1:|Q|], \epsilon', \delta'}}\}
+ \epsilon'.
\end{align}
We can bound the total times \textsc{Merge} is called in a single
instance of $\textsc{Merge-Rank}(S, \epsilon',
\delta')$. \textsc{Merge} combines the singleton sets and forms the
sets with two elements, it combines the sets with two elements and
forms the sets with four elements and henceforth.  Hence the total
times \textsc{Merge} is called is $ \sum_{i=1}^{\log |S|}
\frac{|S|}{2^i} \le |S|.  $ Therefore, the probability that
Equation~\ref{eq:merge_bound} holds every time when two ordered sets
are merged in \textsc{Merge-Rank}$(S, \epsilon', \delta')$ is $ \le
|S|\cdot |S|\delta' = |S|^2\delta'.  $

If Equation~\ref{eq:merge_bound} holds every time \textsc{Merge} is
called, then error of \textsc{Merge-Rank}$(S, \epsilon', \delta')$ is
at most $\sum_{i=1}^{\log |S|} \epsilon' \le \epsilon' \log |S|$. This
is because $err(S)$ is 0 if $S$ has only one element. And a singleton
set participates in $\log n$ merges before becoming the final output
set.

Therefore, w.p. $\ge 1- |S|^2\delta'$,
\[
err(\textsc{Merge-Rank}(S, \epsilon', \delta')) \le \log |S|\epsilon'.
\]
Hence with probability $\ge1-\delta$,
\[
err\Paren{\textsc{Merge-Rank}\Paren{S, \frac{\epsilon}{ \log
|S|}, \frac{\delta}{|S|^2}}} \le \epsilon.\qedhere
\]
\end{proof}

\section{Proofs for Section~\ref{sec:BSR}}
\subsection*{Proof of Lemma~\ref{lem:bin_bound}}
\begin{proof}
Let set $S$ be ordered s.t. $\tilde{p}(S(i),S(j)) \ge 0$ $\forall i >
j$.  Let $S_k''= S(k:k+5 (\log n)^{x+1}-1)$ \ignore{ \{e\in S:
  \tilde{p}(e, S'(k)) > \epsilon , \tilde{p}(S'(k+1), e) < \epsilon\}.
\]
} The probability that none of the elements in $S_k''$ is selected for
a given $k$ is
\[
 \le \Paren{1-\frac{5(\log n)^{x+1}}{n}}^{n/(\log n)^x} < \frac1{n^5}.
\]
Therefore by union bound, the probability that none of the elements in
$S_k''$ is selected for any $k$ is
\[
\le n\cdot \frac1{n^5} =\frac1{n^4}. \qedhere
\]
\end{proof}

\subsection*{Proof of Lemma~\ref{lem:right_bin}}
We prove Lemma~\ref{lem:right_bin} by dividing it into further smaller
lemmas.

We divide all elements into $S$ into two sets based on distance from
anchors. First set contains all elements that are far away from all
anchors and the second set contains all elements which are close to
atleast one of the anchors.  \IBR\ acts differently on both sets.

We first show that for elements in the first set, \IBR\ places them in
between the right anchors by using just the random walk subroutine.

For elements in the second set, \IBR\ might fail to find the right
anchors just by using the random walk subroutine.  But we show that
\IBR\ visits a close anchor during random walk and \BR\ finds a close
anchor from the set of visited anchors using simple binary search.

We first prove Lemma~\ref{lem:right_bin} for the elements of first
set.
\begin{Lemma}
For $\epsilon'' > \epsilon'$, consider an $\epsilon'$-ranked $S'$.  If
an element $e$ is such that $|\tilde{p}(e,S'(j))| > \epsilon''$
$\forall j$, then with probability $\ge 1-\frac1{n^6}$ step
~\ref{output} of \IBR$(S',e,\epsilon'')$ outputs the index $y$ such
that $\tilde{p}(e, S'(y)) > \epsilon''$ and $\tilde{p}(S'(y+1), e) >
\epsilon''$.
\label{lem:far}
\end{Lemma}
\begin{proof}
We first show that there is a unique $y$ s.t.  $\tilde{p}(e, S'(y)) >
\epsilon''$ and $\tilde{p}(S'(y+1), e) > \epsilon''$.

Let $i$ be the largest index such that $\tilde{p}(e, S'(i)) >
\epsilon''$.  By Lemma~\ref{lem:trans_tri}, $\tilde{p}(e, S'(j)) >
\epsilon''- \epsilon' > 0 \quad \forall j < i$.  Hence by the
assumption on $e$, $\tilde{p}(e, S'(j)) > \epsilon'' \quad \forall j <
i$.  Let $k$ be the smallest index such that $\tilde{p}(S'(k), e) >
\epsilon''$.  By a similar argument as previously, we can show that
$\tilde{p}(S'(j),e) > \epsilon'' \quad \forall j > k$.

Hence by the above arguments and the fact that $|\tilde{p}(e, S'(j))| >
\epsilon'' \quad \forall j$, there exists only one $y$ such that
$\tilde{p}(e, S'(y)) > \epsilon''$ and $\tilde{p}(S'(y+1), e) >
\epsilon''$.

Thus in the tree $T$, there is only one leaf node $w$ such that
$\tilde{p}(e, S'(w_1)) > \epsilon''$ and $\tilde{p}(S'(w_2), e) >
\epsilon''$.

Consider some node $m$ which is not an ancestor of $w$. Then
either $\tilde{p}(S'(m_1), e) >
\epsilon''$ or $\tilde{p}(S'(m_2), e) < - \epsilon''$. Since we compare
$e$ with $S'(m_l)$ and $S'(m_h)$ $\frac{10}{\epsilon''^2}$ times, we
move to the parent of $m$ with probability atleast $\frac{19}{20}$.

Consider some node $m$ which is an ancestor of $w$. Then
$\tilde{p}(S'(m_l), e) < - \epsilon''$ , $\tilde{p}(S'(m_h), e) >
\epsilon''$, and
$|\tilde{p}(S'(\left\lceil\frac{m_l+m_h}2\right\rceil), e)| >
\epsilon''$. Therefore we move in direction of $q$ with probability
atleast $\frac{19}{20}$.

Therefore if we are not at $q$, then we move towards $q$ with
probability atleast $\frac{19}{20}$ and if we are at $q$ then the
count $c$ increases with probability atleast $\frac{19}{20}$.

Since we start at most $ \log n$ away from $q$ if we move towards $e$
for $21 \log n$ then the algorithm will output $y$. The probability
that we will have less than $21 \log n$ right comparisons is $\le
e^{-30 \log n D(\frac{21}{30}||\frac{19}{20})} \le e^{-30 \log n
  D(\frac{21}{30}||\frac{19}{20})} \le \frac1{n^6}$.
\end{proof}

To prove Lemma~\ref{lem:right_bin} for the elements of the second set,
we first show that the random walk subroutine of algorithm
\IBR\ placing an element in wrong bin is highly unlikely.
\begin{Lemma}
For $\epsilon'' > \epsilon'$, consider an $\epsilon'$-ranked set
$S'$. Now consider an element $e$ and $y$ such that either
$\tilde{p}(S'(y), e) > \epsilon''$ or $\tilde{p}(S'(y+1), e) < -
\epsilon''$, then step ~\ref{output} of \IBR$(S', e, \epsilon'')$ will
not output $y$ with probability $\ge 1-\frac1{n^7}$.
\label{lem:interval}
\end{Lemma}

\begin{proof}
 Recall that step ~\ref{output} of \IBR\ outputs $y$ if we are at the
 leaf node $(y, y+1)$ and the count $c$ is atleast $10 \log n$.
 
 Since either $\tilde{p}(S'(y), e) > \epsilon''$ or
 $\tilde{p}(S'(y+1), e) < - \epsilon''$, when we are at leaf node
 $(y,y+1)$, the count decreases with probability atleast
 $\frac{19}{20}$. Hence the probability that \IBR\ is at (y,y+1) and
 the count is greater than $10 \log n$ is at most $ \sum_{i=10\log
   n}^{30 \log n} e^{-i\cdot D(\frac{i - 10\log n}{2i} ||
   \frac{19}{20})} < 10 \log n e^{-10\log n D(\frac13 ||
   \frac{19}{20})} \le \frac1{n^7}$.
\end{proof}

We now show that for an element of the second set, the random walk
subroutine either places it in correct bin or visits a close anchor.

\begin{Lemma}
 For $\epsilon'' > \epsilon'$, consider an $\epsilon'$-ranked set
 $S'$.  Now consider an element $e$ that is close to an element in
 $S'$ i.e., $\exists g : | \tilde{p}(S'(g), e) | < \epsilon''$. Step
 ~\ref{output} of \IBR$(S',e,\epsilon'')$ will either output the right
 index $y$ such that $\tilde{p}(S'(y),e) < \epsilon''$ and
 $\tilde{p}(S'(y+1),e) > -\epsilon''$ or \IBR\ visits $S'(h)$
 such that $ |\tilde{p}(S'(h), e) | < 2\epsilon''$ with
 probability$\ge 1- \frac1{n^6}$.
\label{lem:close}
\end{Lemma}
\begin{proof}
 By Lemma~\ref{lem:interval}, step ~\ref{output} of \IBR\ does not
 output a wrong interval with probability $1- \frac{1}{n^{7}}$. Hence
 we just need to show that w.h.p., $e$ visits a close anchor.

Let $i$ be the largest index such that $\tilde{p}(e, S'(i)) >
2\epsilon''$. Then $\forall k < i$, by Lemma~\ref{lem:trans_tri},
$\tilde{p}(e, S(k)) > 2\epsilon'' - \epsilon' > \epsilon''$ .

Let $j$ be the smallest index such that $\tilde{p}(S'(j), e) >
2\epsilon''$. Then $\forall k > j$, by Lemma~\ref{lem:trans_tri},
$\tilde{p}(S'(k),e) > \epsilon''$ .

Therefore for $u<v$ such that $\min(|\tilde{p}(S'(u),e)|,
|\tilde{p}(S'(v),e)| ) \ge 2\epsilon''$ only one of three sets $\{x:
x< u\}$,$\{x: u<x<v\}$ and $\{x: x>v\}$ contains an index $z$ such
that $|\tilde{p}(S'(z), e)| < \epsilon''$.

Let a node $\alpha$ be s.t.  for some $c \in \{\alpha_1, \alpha_2,
\ceil{\frac{\alpha_1+\alpha_2}2}\}$, $|\tilde{p}(S'(c), e)| \le 2\epsilon''$. If 
\IBR\ reaches such a node $\alpha$ then we are done.

So assume that \IBR\ is at a node $\beta$ s.t. $\forall c \in
\{\beta_1, \beta_2, \ceil{\frac{\beta_1+\beta_2}2}\}$,
$|\tilde{p}(S'(c), e)| > 2\epsilon''$.  Note that only one of three
sets $\{x:x < \beta_1 \text{ or } x > \beta_2\}$, $\{x: \beta_1 < x <
\ceil{\frac{\beta_1+\beta_2}2}\}$ and $\{x:
\ceil{\frac{\beta_1+\beta_2}2} < x < \beta_2\}$ contains an index $z$
such that $|\tilde{p}(S'(z), e)| < \epsilon''$ and \IBR\ moves towards
that set with probability $\frac{19}{20}$. Hence the probability that
we never visit an anchor that is less than $2\epsilon''$ away is at
most $ e^{-30\log n D(\frac12 || \frac{19}{20})} \le \frac1{n^{7}}$.
\end{proof}

We now complete the proof by showing that for an element $e$ from the
second set, if $Q$ contains an index $y$ of an anchor that is close to
$e$, \BR\ will output one such index.

\begin{Lemma}
 For $\epsilon'' > \epsilon'$, consider ordered sets $S',Q$
 s.t. $p(S'(Q(i)), S'(Q(j))) > \frac12-\epsilon'$ $\forall i > j$. For
 an element $e$ s.t., $\exists
 g : |\tilde{p}(S'(Q(g)),a)| < 2\epsilon''$,  \BR$(S',Q,a,\epsilon'')$
 will return $y$ such that $|\tilde{p}(S'(Q(y)),a)| < 4\epsilon''$
 with probability $\ge 1- \frac1{n^6}$.
\label{lem:close2}
\end{Lemma}
\begin{proof}

At any stage of \BR, there are three possibilities that can happen .
Consider the case when we are comparing $e$ with $S'(Q(i))$.

1. $|\tilde{p}(S'(Q(i)),e) < 2\epsilon''|$. Probability that the
fraction of wins for $e$ is not between $\frac12 - 3\epsilon''$ and 
$\frac12 + 3\epsilon''$ is less than $e^{-\frac{10\log
    n}{\epsilon''^2}\epsilon''^2} \le \frac{1}{n^{10}}$. Hence
\BR\ outputs $Q(i)$.

2. $\tilde{p}(S'(Q(i)),e)>2\epsilon''$. Probability that the fraction
of wins for $e$ is more than $\frac12$ is less than $e^{-\frac{10\log
    n}{\epsilon''^2}\epsilon''^2} \le \frac1{n^10}$.  So
\BR\ will not move right. Also notice that
$\tilde{p}(S'(Q(j)),e) > 2\epsilon'' -\epsilon' > \epsilon''$ $\forall
j > i$.

3. $\tilde{p}(S'(Q(i)), e) > 4\epsilon''$. Probability that the
fraction of wins for $e$ is more than $\frac12 - 3\epsilon''$ is less
than $e^{-\frac{10\log n}{\epsilon''^2}\epsilon''^2} \le
\frac1{n^{10}}$.  Hence \BR\ will move
left. Also notice that $\tilde{p}(S'(Q(j)), e) > 4\epsilon'' -
\epsilon' > \epsilon''$ $\forall j > i$.

We can show similar results for $\tilde{p}(S'(Q(i)), e) <
-2\epsilon''$ and $\tilde{p}(S'(Q(i)), e) < -4\epsilon''$.  Hence if
$|\tilde{p}(S'(Q(i)),e)| < 2 \epsilon''$ then \BR\ outputs $Q(i)$, and
if $2 \epsilon'' <|\tilde{p}(S'(Q(i)),e)| < 4 \epsilon''$ then either
\BR\ outputs $Q(i)$ or moves in the correct direction and if
$|\tilde{p}(S'(Q(i)),e)| > 4 \epsilon''$, then \BR\ moves in the correct
direction.
\end{proof}

\begin{Lemma}
\label{lem:ibr_complexity}
\IBR$(S, e, \epsilon)$ terminates in $\cO(\frac{\log n \log \log
  n}{\epsilon^2})$ comparisons for any set $S$ of size $O(n)$.
\end{Lemma}
\begin{proof}
Step 3 of \IBR\ runs for 30 $\log n$ iterations. In each iteration,
\IBR\ compares $e$ with at most 3 anchors and repeats each comparison for 
$10/\epsilon^2$. So total comparisons in step $3$ is $\cO(\log n/ \epsilon^2)$.
The size of $Q$ is upper bounded by $90 \log n$ and \BR\ does a simple
binary search over $Q$ by repeating each comparison $10\log n/ \epsilon^2$.
Hence total comparisons used by \BR\ is $\cO(\log n \log \log n /\epsilon^2)$  
\end{proof}

Combining Lemmas
~\ref{lem:anchor_ranking},~\ref{lem:interval},~\ref{lem:close},~\ref{lem:close2},~\ref{lem:ibr_complexity}
yields the result.

\subsection*{Proof of Lemma~\ref{lem:real_bin_bound}}

\begin{proof}
Combining Lemmas ~\ref{lem:anchor_ranking},~\ref{lem:middle_bin} and
using union bound, at the end of step ~\ref{step:binning_further}
,w.p. $\ge 1-\frac2{n^3}$, $S'$ is $\epsilon'$-ranked and $\forall j,
e \in B_j$, $\min(\tilde{p}(e, S'(j)),\tilde{p}(S'(j+1), e)) > 5
\epsilon''$. Hence by Lemma~\ref{lem:trans_tri}, $\forall j,k<j, e \in
B_j$, $\tilde{p}(e, S'(k)) > 5\epsilon'' - \epsilon' > 4 \epsilon''$.
Similarly, $\forall j,k>j, e \in B_j$, $\tilde{p}( S'(k),e) >
5\epsilon'' - \epsilon' > 4 \epsilon''$.

If $|B_j| > 0$, then $\tilde{p}(e, S'(k)) > 4 \epsilon''$ for $e \in
B_j, k \le j$, $\tilde{p}(S'(l), e) > 4 \epsilon''$ for $e \in B_j, l
\ge j$. Hence by strong stochastic transitivity, $\tilde{p}(S'(l),
S'(k)) > 4 \epsilon''$ for $l \ge j \ge k$.  Therefore there exists
$k,l$ s.t. $\tilde{p}(S'(l),f) > 0$ $\forall f \in \{S'(y):y \le j\}$,
$\tilde{p}(S'(k), S'(l)) > 0$ and $\tilde{p}(f, S'(k)) > 0$ $\forall f
\in \{S'(y): y \ge j\}$. Now by Lemma ~\ref{lem:bin_bound}, w.p. $\ge
1- \frac1{n^3}$, size of such set $B_j$ is less than $10 (\log
n)^{x+1}$.

Lemma follows by union bound.
\end{proof}

\subsection*{Proof of Theorem~\ref{thm:ranking_upper}}
We first bound the running time of \textsc{Binary-Search-Ranking}
algorithm.
\begin{Theorem}
\label{thm:bsr_complexity}
\textsc{Binary-Search-Ranking} terminates after $O(\frac{n (\log \log
  n)^x}{\epsilon^2} \log n )$ comparisons with probability $\ge
1-\frac1{n^2}$.
\end{Theorem}
\begin{proof}
Step \ref{step:anchor_ranking} \RankX\ $(S',\epsilon',
\frac{1}{n^6})$ terminates after $\cO(\frac{n}{\epsilon^2} \log n)$
comparisons with probability $\ge 1- \frac1{n^6}$.

By Lemma ~\ref{lem:right_bin}, for each element $e$, the step
\ref{step:binning_element} \IBR$(S', e, \epsilon'')$ terminates after
$\cO(\frac{\log n \log \log n}{\epsilon^2} \log \log n)$
comparisons. Hence step \ref{step:binning} takes at most $\cO( \frac{n
  \log n \log \log n}{\epsilon^2})$ comparisons.

Comparing each element with the anchors in steps
\ref{step:binning_further} takes at most $O(\frac{\log
  n}{\epsilon^2})$ comparisons.

With probability $\ge 1-\frac1{n^4}$ step \ref{step:bin_ranking}
\textsc{Sort-x}$(B_i, \ \epsilon'', \frac1{n^4})$ terminates after
$\cO(|B_i|\frac{(\log |B_i|)^x}{\epsilon^2} \log n)$ comparisons. By
Lemma ~\ref{lem:real_bin_bound}, $|B_i| \le 10 (\log n)^{x+1}$ for all
$i$ w.p. $\ge 1 - \frac3{n^3}$. Hence, w.p. $\ge 1- \frac3{n^3}$,
total comparisons to rank all $B_i$s is at most $ \sum_i \cO(|B_i|
\frac{(\log |B_i|)^x}{\epsilon^2} \log n) \le \sum_i
\cO(\frac{|B_i|\ log n(\log (10 \log n)^{x+1})^x}{\epsilon^2}) =
\cO(\frac{n \log n (\log \log n)^x}{\epsilon^2}).  $

Therefore, by summing comparisons over all steps, with probability
$\ge 1-\frac1{n^2}$ total comparisons is at most $O\Paren{\frac{n \log
    n (\log \log n)^x}{\epsilon^2} }$.
\end{proof}

Now we show that \textsc{Binary-Search-Ranking} outputs an
$\epsilon$-ranking with high probability.
\begin{Theorem}
\label{thm:bsr_accuracy}
\textsc{Binary-Search-Ranking} produces an $\epsilon$-ranking with
probability at least $1 - \frac1{n^2}$.
\end{Theorem}

\begin{proof}
By combining
Lemmas~\ref{lem:anchor_ranking},~\ref{lem:boundary_bin},~\ref{lem:middle_bin},~\ref{lem:bin_ranking_accuracy}
and using union bound, w.p. $\ge 1-\frac1{n^2}$, at the end of
step~\ref{step:bin_ranking},
\begin{itemize}
\item $S'$ is $\epsilon'$-
\item Each $C_i$ has elements such that $|\tilde{p}(C_i(j), S(i))| <
   7 \epsilon''$ for all $j$.

\item Each $B_i$ has elements such that $\tilde{p}(S'(i), B_i(j)) < -5
  \epsilon''$ and $\tilde{p}(S'(i+1), B_i(j)) > 5 \epsilon''$.

 \item All $B_i$s are $\epsilon''$-ranked.
\end{itemize}
For $j \ge i$, $e \in B_{j-1} \bigcup {S'(j)}\bigcup C_j$, $f \in
{S'(k)} \bigcup C_k \bigcup B_k$, $\tilde{p}(e,f) \le \tilde{p}(e,
S'(i)) + \tilde{p}(S'(i), S'(j)) + \tilde{p}(S'(j), f) \le 7
\epsilon'' + \epsilon' + 7 \epsilon'' < 15 \epsilon'' = \epsilon.  $
Combining the above result with the fact that all $B_i$s are
$\epsilon''$-ranked proves the Lemma.
\end{proof}

Combining Theorems ~\ref{thm:bsr_complexity},~\ref{thm:bsr_accuracy}
yields the result.

\subsection*{Proof Sketch for Theorem~\ref{thm:ranking_lower}}
\begin{proof}[Proof sketch]
Consider a stochastic model where there is an inherent ranking $r$ and
for any two consecutive elements $p(i,i+1) = \frac{1}{2} - 2\epsilon$.
Suppose there is a genie that knows the true ranking $r$ up to the
sets $\Sets{r(2i-1), r(2i)}$ for all $i$ i.e., for each $i$, genie
knows $\{r(2i-1), r(2i)\}$ but it does not know the ranking between
these two elements. Since consecutive elements have $\epsilon(i,i+1) =
2\epsilon > \epsilon$, to find an $\epsilon$-ranking, the genie has to
correctly identify the ranking within all the $n/2$ pairs.  Using
Fano's inequality from information theory, it can be shown that the
genie needs at least $\Omega\Paren{\frac{n}{\epsilon^2} \log
  \frac{n}{\delta}}$ comparisons to identify the ranking of the
consecutive elements with probability $1-\delta$.
\end{proof}

\section{Additional Experiments}
\label{sec:app_exp}

As we mentioned in Section~\ref{sec:experiments}, \textbf{BTM-PAC}
allows comparison of an element with itself.  It is not beneficial
when the goal is to find $\epsilon$-maximum.  So we modify their
algorithm by not allowing such comparisons.  We refer to this
restricted version as \textbf{R-BTM-PAC}.

As seen in figure, performance of \textbf{BTM-PAC} does not increase
by much by restricting the comparisons.

We further reduce the constants in \textbf{R-BTM-PAC}. We change
Equations (7) and (8) in~\cite{YisongT11} to $c_{\delta}(t) =
\sqrt{\frac1{t}\log \frac{n^3 N}{\delta}}$ and $N =
\ceil{\frac1{\epsilon^2} \log \frac{n^3N}{\delta}}$, respectively.

We believe the same guarantees hold even with the updated constants.
We refer to this improved restricted version as \textbf{IR-BTM-PAC}.
Here too we consider the stochastic model where $p(i, j) = 0.6
\forall\ i < j$ and we find $0.05$-maximum with error probability
$\delta = 0.1$.

In Figure~\ref{fig:7} we compare the performance of
\textsc{Knockout} and all variations of \textsc{BTM-PAC}. As the
figure suggests, the performance of \textbf{IR-BTM-PAC} improves a lot
but \textsc{Knockout} still outperforms it significantly.

\begin{figure}[H]
\centering
\includegraphics[scale=0.4]{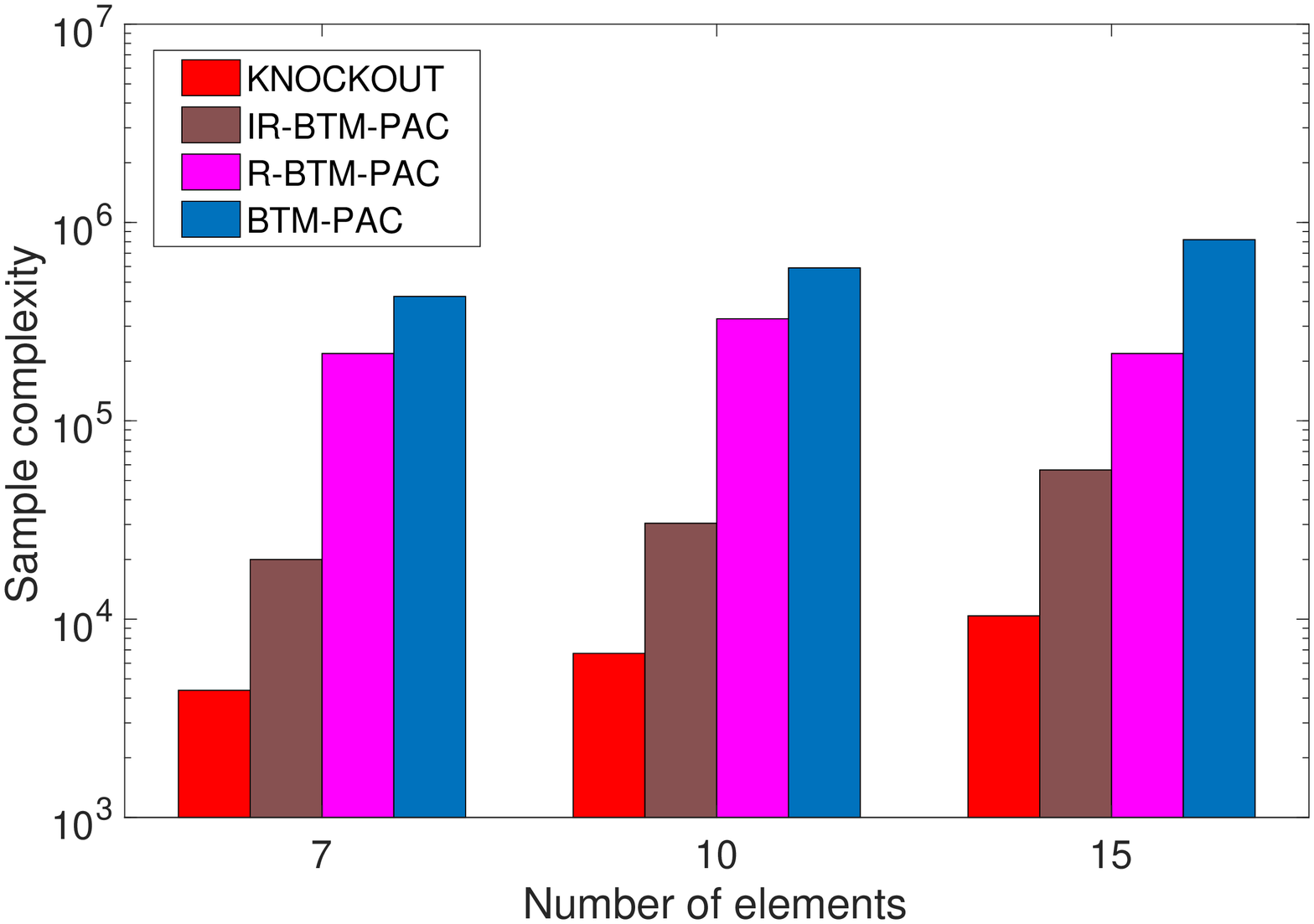}
\caption{Sample complexity comparison of \textsc{Knockout} and variations of BTM-PAC
  for different input sizes, with $\epsilon = 0.05$ and $\delta = 0.1$}
\label{fig:7}
\end{figure}

\begin{figure}[H]
\centering
\includegraphics[scale=0.4]{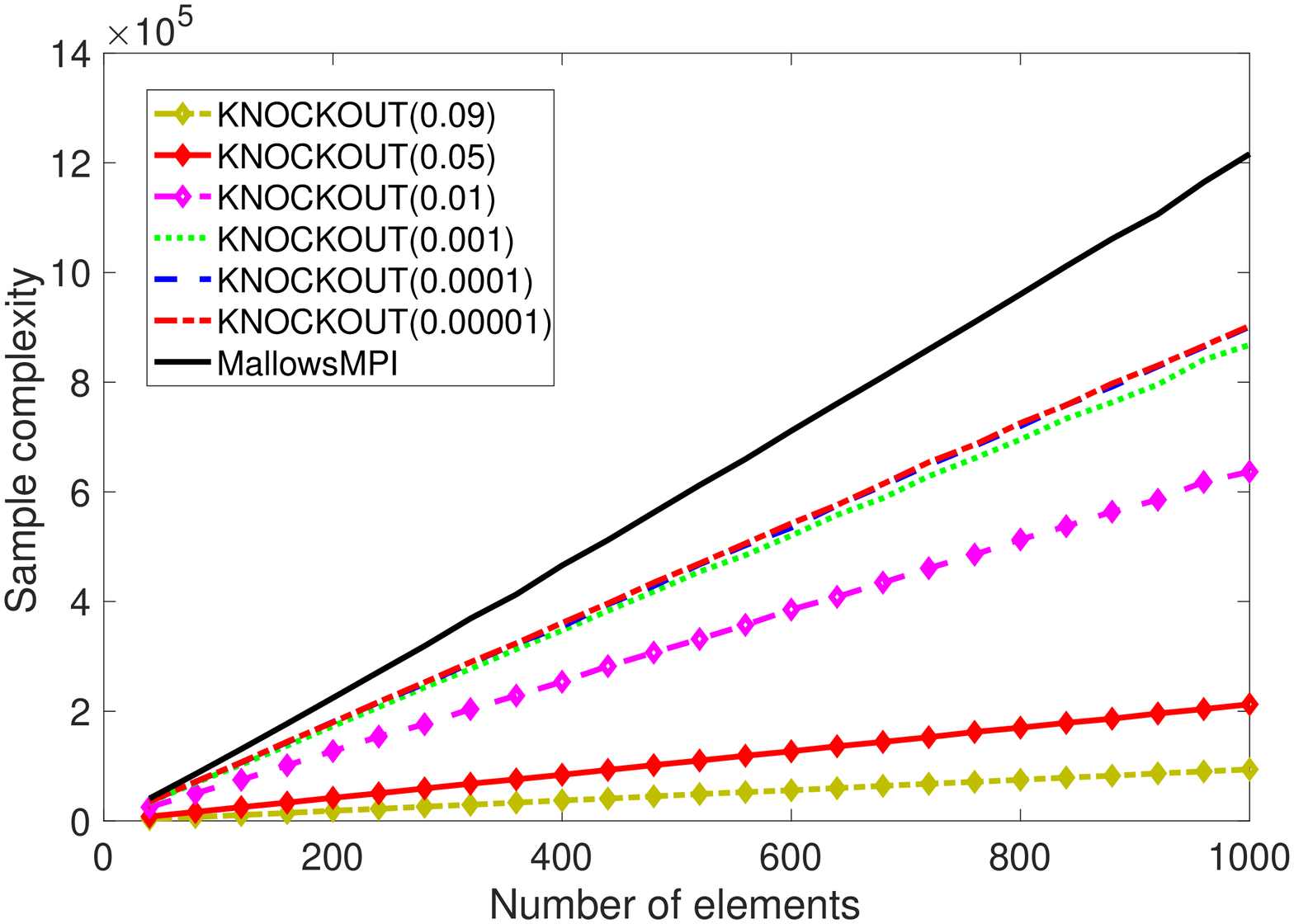}
\caption{Sample complexity of \textsc{Knockout} for different values
  of $n$ and $\epsilon$}
\label{fig:4}
\end{figure}

In Figure~\ref{fig:4}, we consider the stochastic model where
$p(i,j) = 0.6\ \forall i < j$ and find $\epsilon$-maximum for
different values of $\epsilon$. Similar to previous experiments, we
use $\delta = 0.1$. As we can see the number of comparisons increases
almost linearly with $n$. Further the number of comparisons does not
increase significantly even when $\epsilon$ decreases. Also the number
of comparisons seem to be converging as $\epsilon$ goes to
0. \textsc{Knockout} outperforms \textbf{MallowsMPI} even for the very
small $\epsilon$ values. We attribute this to the subroutine
\textsc{Compare} that finds the winner faster when the distance between 
elements are much larger than $\epsilon$.

\begin{figure}[H]
\centering
\includegraphics[scale=0.4]{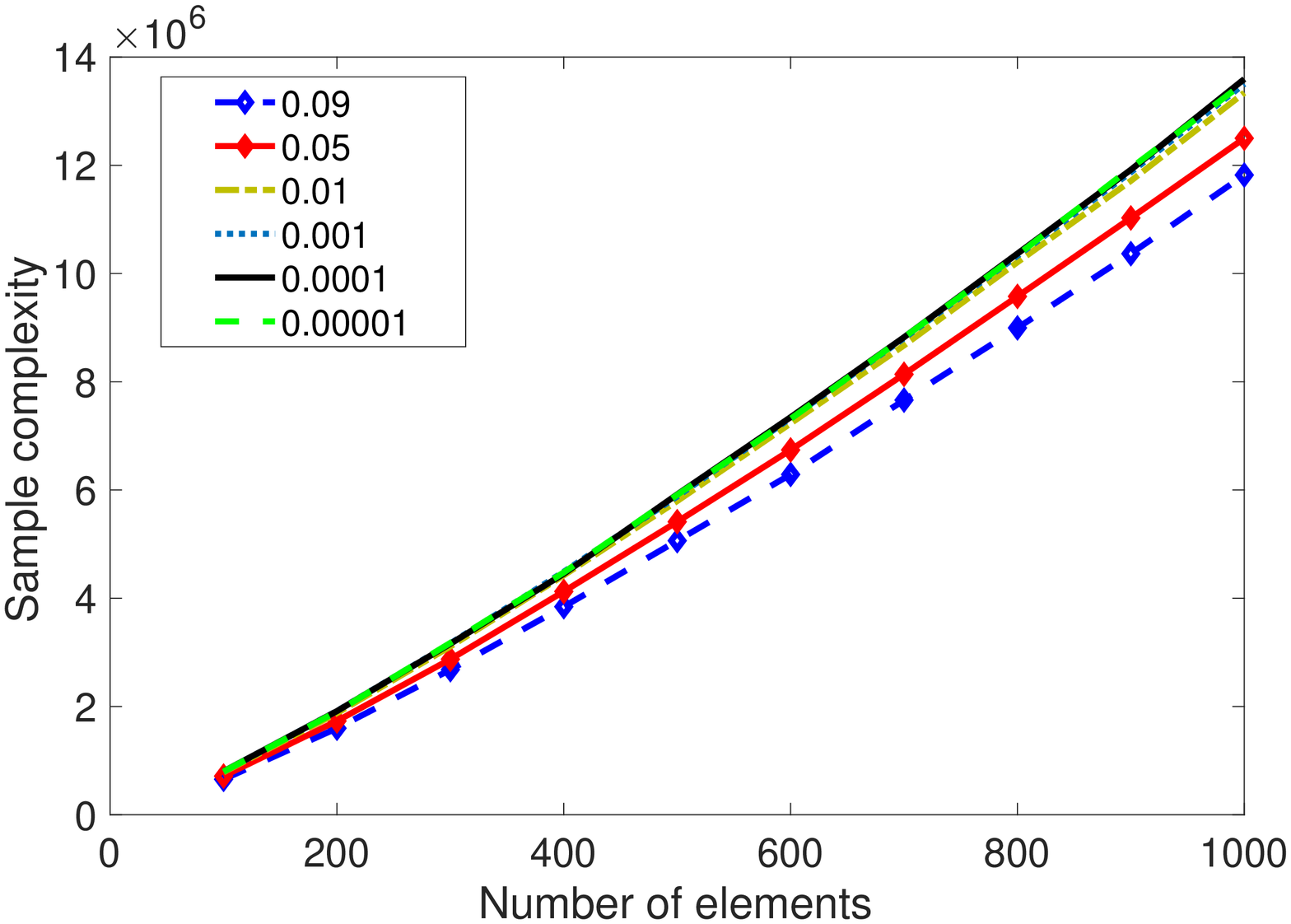}
\caption{Sample complexity of \textsc{Merge-Rank} for $\epsilon$}
\label{fig:5}
\end{figure}

For the stochastic model $p(i,j) = 0.6\ \forall i <j$, we 
run our \textsc{Merge-Rank} algorithm to find $0.05$ ranking 
with $\delta = 0.1$. Figure~\ref{fig:5} shows that sample complexity does not
increase a lot with decreasing $\epsilon$.

\end{document}